\documentclass[11pt]{article}
 
\usepackage[sort,numbers]{natbib}

\usepackage[bookmarks=true,bookmarksnumbered=true,colorlinks=true,allcolors=black]{hyperref}
\usepackage{amsmath,amssymb,amsfonts,amsthm}
\usepackage{lscape}
\usepackage{arydshln}
\renewcommand*{\baselinestretch}{1.25}

\usepackage[capitalize]{cleveref}

\usepackage{geometry}
\usepackage{indentfirst}

\usepackage{enumitem,color}

\newtheorem{theorem}{Theorem}[section]
\newtheorem{lemma}{Lemma}[section]
\newtheorem{proposition}{Proposition}[section]
\newtheorem{corollary}{Corollary}[section]
\theoremstyle{definition}
\newtheorem{definition}{Definition}[section]
\newtheorem*{rmk*}{Remark}
\newtheorem{rmk}{Remark}[section]

\newtheorem{example}{Example}[section]

\DeclareMathOperator{\covariance}{Cov}

\DeclareMathOperator{\trace}{tr}

\DeclareMathOperator{\E}{E}

\DeclareMathOperator{\tv}{TV}

\allowdisplaybreaks

\numberwithin{equation}{section}

\makeatletter
    \renewcommand*{\section}{\@startsection{section}{1}{\z@}%
    {6pt}{3pt}{\reset@font\normalsize\bfseries}}
\makeatother
    
\makeatletter
    \renewcommand*{\subsection}{\@startsection{subsection}{2}{\z@}%
    {3pt}{3pt}{\reset@font\normalsize\mdseries\itshape}}
\makeatother

\makeatletter
    \renewcommand*{\subsubsection}{\@startsection{subsubsection}{3}{\z@}%
    {3pt}{3pt}{\reset@font\normalsize\mdseries\itshape}}
\makeatother

\makeatletter
\def\@listi{\leftmargin\leftmargini
  \topsep=.5\baselineskip 
  \partopsep=0pt \parsep=0pt \itemsep=0pt}
\let\@listI\@listi
\@listi
\def\@listii{\leftmargin\leftmarginii
  \labelwidth\leftmarginii \advance\labelwidth-\labelsep
  \topsep=0pt \partopsep=0pt \parsep=0pt \itemsep=0pt}
\def\@listiii{\leftmargin\leftmarginiii
  \labelwidth\leftmarginiii \advance\labelwidth-\labelsep
  \topsep=0pt \partopsep=0pt \parsep=0pt \itemsep=0pt}
\def\@listiv{\leftmargin\leftmarginiv
  \labelwidth\leftmarginiv \advance\labelwidth-\labelsep
  \topsep=0pt \partopsep=0pt \parsep=0pt \itemsep=0pt}
\makeatother

\makeatletter
\newcommand{\opnorm}{\@ifstar\@opnorms\@opnorm}
\newcommand{\@opnorms}[1]{%
  \left|\mkern-1.5mu\left|\mkern-1.5mu\left|
   #1
  \right|\mkern-1.5mu\right|\mkern-1.5mu\right|
}
\newcommand{\@opnorm}[2][]{%
  \mathopen{#1|\mkern-1.5mu#1|\mkern-1.5mu#1|}
  #2
  \mathclose{#1|\mkern-1.5mu#1|\mkern-1.5mu#1|}
}
\makeatother

\def\be#1{\begin{equation*}#1\end{equation*}}
\def\ben#1{\begin{equation}#1\end{equation}}

\def\besn#1{\begin{equation}\begin{split}#1\end{split}\end{equation}}

\def\ba#1{\begin{align*}#1\end{align*}}
\def\ban#1{\begin{align}#1\end{align}}

\newcommand{\eps}{\varepsilon}

\newcommand{\score}{\wh{\mf s}}
\newcommand{\scored}{\mathsf{s}}
\newcommand{\wiener}{\mathbb{W}}
\newcommand{\slepian}{\mathcal{S}}
\newcommand{\euler}{X^\mathrm{EM}}
\newcommand{\ada}{X^\mathrm{ada}}
\newcommand{\ddpm}{\mathsf{x}}

\newcommand{\ol}[1]{\overline{#1}}

\newcommand{\wh}[1]{\widehat{#1}}
\newcommand{\wt}[1]{\widetilde{#1}}
\newcommand{\mf}[1]{\mathfrak{#1}}
\newcommand{\mcl}[1]{\mathcal{#1}}

\newcommand{\bra}[1]{\left(#1\right)}
\newcommand{\cbra}[1]{\left\{#1\right\}}
\newcommand{\sbra}[1]{\left[#1\right]}
\newcommand{\abs}[1]{\left|#1\right|}

\makeatletter
\newcommand{\pushright}[1]{\ifmeasuring@#1\else\omit\hfill$\displaystyle#1$\fi\ignorespaces}
\newcommand{\pushleft}[1]{\ifmeasuring@#1\else\omit$\displaystyle#1$\hfill\fi\ignorespaces}
\makeatother

\title{Connections between the F\"ollmer process and the denoising diffusion probabilistic model}
\author{Yuta Koike\thanks{Graduate School of Mathematical Sciences, University of Tokyo}
\thanks{CREST, Japan Science and Technology Agency}}

\begin{document}

\maketitle

\begin{abstract}
The F\"ollmer process is a Brownian motion conditioned to have a pre-specified distribution at time 1. 
This process can be interpreted as an ``augmented'' time-compressed version of the reverse stochastic differential equation (SDE) corresponding to the denoising diffusion probabilistic model (DDPM). 
While this fact has been indirectly used to analyze DDPM sampling errors via discretization of the reverse SDE, the connection between direct discretization of the F\"ollmer process and the DDPM sampler has not yet been fully explored. 
This paper clarifies this point while surveying relevant results from the literature. 
We show that discretized F\"ollmer processes give natural hyper-parameter settings of the DDPM sampler while accommodating a broader class of variance schedules than discretized reverse SDEs. 
Moreover, this allows us to systematically recover state-of-the-art results on DDPM sampling error bounds, along with slight improvements. 
\vspace{2mm}


\end{abstract}

\section{Introduction}

Let $\mu$ be a probability distribution on $\mathbb R^d$. 
Following \cite{ELS20}, we call a $d$-dimensional continuous stochastic process $X=(X_t)_{t\in[0,1]}$ the \textit{F\"ollmer process} associated with $\mu$ if $X_1\sim\mu$ and the conditional law of $X$ given $X_1$ is the $d$-dimensional Brownian bridge from 0 to $X_1$ on $[0,1]$. 
This process has recently received renewed attention in various fields of both theory and application such as studies of functional inequalities \cite{Le13,ElMi20,ELS20,ElLe18,MiSh24,Mik21,KlLe25} as well as high-dimensional central limit theorems \cite{EMZ20,FaKo24,FXZ26}, optimization \cite{DJKLY23,MFV24}, generative modeling \cite{CGHABV24,WJXWY21}, sampling from probability distributions \cite{TzRa19,HJKLLL25,ZhCh22,EnNa24} and Bayesian inference \cite{VOFGLN23,RBCJK23}. 
Historically, the F\"ollmer process appeared as the solution to an entropy minimization problem posed by Schr\"odinger. 
A modern formulation of Schr\"odinger's problem is as follows (see \cite{Le14,Mi21} for more details). 
Let $\mathbb W^d:=C([0,1],\mathbb R^d)$ be the space of continuous functions from $[0,1]$ to $\mathbb R^d$ equipped with the uniform topology. 
For a probability distribution $Q$ on $\mathbb W^d$, we denote by $Q_0$ and $Q_1$ the image measures of $Q$ by the maps $w\mapsto w(0)$ and $w\mapsto w(1)$, respectively.   
Then, given a pair $(\mu_0,\mu_1)$ of probability distributions on $\mathbb R^d$ and a probability distribution $R$ on $\mathbb W^d$, the \emph{dynamical Schr\"odinger problem} associated with the reference measure $R$ is the following minimization problem with respect to probability distributions $Q$ on $\mathbb W^d$:
\ben{\label{dyn-schr}
\min_{Q:Q_0=\mu_0,Q_1=\mu_1}H(Q\mid R).
} 
Here, given two probability measures $P$ and $Q$ defined on a common measurable space, $H(P\mid Q)$ denotes the relative entropy of $P$ with respect to $Q$, i.e.
\[
H(P\mid Q)=\int\frac{dP}{dQ}\log\bra{\frac{dP}{dQ}}dQ\quad\text{if }P\ll Q
\]
and $H(P\mid Q)=\infty$ otherwise. 
\citet{Fo88} showed that when $H(\mu\mid\gamma_d)<\infty$ with $\gamma_d=N(0,I_d)$ the $d$-dimensional standard normal distribution, the law of $X$ gives the unique solution to \eqref{dyn-schr} with $\mu_1=\mu$, $\mu_0$ being the unit mass at 0 and $R$ being the $d$-dimensional Wiener measure; see Section II.1.3 ibidem. 
%
%
Furthermore, under the finite entropy condition $H(\mu\mid\gamma_d)<\infty$, it follows from \citet{Fo86}'s results that $X$ is a weak solution to the following stochastic differential equation (SDE):
\ben{\label{eq:follmer}
dX_t=\bra{\frac{X_t}{t}+\nabla\log p_t(X_t)}dt+dW_t\quad(0\leq t\leq1),\qquad X_0=0,
}
where $W=(W_t)_{t\in[0,1]}$ is a standard Wiener process in $\mathbb R^d$ and for every $0<t<1$, $p_t$ denotes the density of $X_t$, i.e.
\[
p_t(x)=\frac{1}{\{2\pi t(1-t)\}^{d/2}}\int_{\mathbb R^d}\exp\bra{-\frac{|x-ty|^2}{2t(1-t)}}\mu(dy),\quad x\in\mathbb R^d.
\]
See \cref{follmer-work} for details and also \cite{Ja75,DaP91} for related results. 
As the first main result of this paper, we show in \cref{thm:weak} that $X$ always solves \eqref{eq:follmer} without \emph{any} assumption on $\mu$.

The function $\nabla \log p_t$ is known as the \textit{score function} of (the law of) $X_t$, and  the form of \eqref{eq:follmer} reminds us of the SDE formulation of \emph{score-based generative modeling} (SGM) in \citet{SSKKEP21}. 
In fact, it turns out that $X$ is an ``augmented'' time-compressed version of the reverse SDE corresponding to the \textit{denoising diffusion probabilistic model} (DDPM) of \cite{HJA20}; see \cref{thm:rou}. 
Although this fact is essentially known in the literature (see the discussion preceding \cref{thm:rou}), existing work does not seem to fully explore connections between discretization of the F\"ollmer process and the DDPM sampler. 
This paper aims to clarify this point while surveying relevant results in the literature. 
As we will see in \cref{sec:main}, discretized F\"ollmer processes give natural hyper-parameter settings of the DDPM sampler. 
More importantly, the resulting parametrization does not require the variance schedule of DDPM to be uniformly small, in contrast to the SDE formulation of \cite{SSKKEP21} (see Appendix E ibidem). 
This effectively broadens the class of variance schedules under consideration; see Examples \ref{ex:cos} and \ref{ex:geom}, along with the discussion after \eqref{eq:ddpm}.
In addition, we can systematically recover state-of-the-art results on DDPM sampling error bounds in \cite{BDDD24,CDGS24,HWC26}, along with slight improvements. 
Notably, by exploiting the ``augmented'' part of the F\"ollmer process, we can eliminate the dependence on the ambient dimension in the initialization error of the DDPM sampler; see \cref{coro:hwc} and \cref{rmk:hwc}.
The key to our development is the fact that the F\"ollmer process is strongly linked to some information-theoretic quantities and \citet{El13}'s stochastic localization. 
Indeed, these links were implicitly used in \cite{BDDD24,HWC26} to analyze discretization errors of the reverse SDE in SGM. 
Our development complements these works. 

While the F\"ollmer process has been already used for generative modeling as mentioned above, most studies have used it for developing an alternative model to the DDPM rather than analyzing the DDPM itself. 
In fact, \citet{CGHABV24} and \citet{WJXWY21} employ the F\"ollmer process as a part of more sophisticated generative models proposed therein. 
We also remark that a solution to the problem \eqref{dyn-schr} in the general case, known as the \emph{Schr\"odinger bridge}, is extensively employed for generative modeling. 
Among others, its connection to SGM is discussed in \cite{de2021diffusion,CLT22,CVE26,Pe23}. 
However, the focus of these studies is to develop a more general framework than SGM. Note that \citet{CVE26} use the term ``F\"ollmer process'' to indicate the Schr\"odinger bridge whose initial marginal $\mu_0$ is the Dirac measure.





The remainder of this paper is organized as follows. 
\cref{sec:follmer} reviews basic properties of the F\"ollmer process. 
\cref{sec:main} discusses its connection to the DDPM, with emphasis on its implication to sampling error evaluation. 
The appendix presents additional mathematical proofs. 

\paragraph{Notation}

For a vector $x\in\mathbb R^d$, we denote the Euclidean norm of $x$ by $|x|$. 
If $\xi$ is a random vector in $\mathbb R^d$, $P^\xi$ denotes the law of $\xi$. 
For a positive definite symmetric matrix $\Sigma$, we write $\phi_\Sigma$ for the density of $N(0,\Sigma)$. 
Given a $d$-dimensional stochastic process $X=(X_t)_{t\in I}$ with $I$ an interval in $[0,\infty)$, we define a filtration $\mathbf F^X=(\mcl F^X_t)_{t\in I}$ by $\mcl F^X_t:=\sigma(X_s:s\in I,s\leq t)\vee\mcl N$, where $\mcl N$ is the class of null sets with respect to the underlying probability measure. 
If $I=[0,1]$ and $X$ is continuous, we denote by $P^X$ the path measure of $X$ induced on $\mathbb W^d$. 
In particular, $P^W$ denotes the $d$-dimensional Wiener measure. 
For two numeric sequences $a_n$ and $b_n$, we write $a_n\asymp b_n$ if $a_n=O(b_n)$ and $b_n=O(a_n)$.





\section{Properties of the F\"ollmer process}\label{sec:follmer}

In this section, we review basic properties of the F\"ollmer process. 
Although most results are known in the literature at least under additional regularity on $\mu$, we give their proofs to avoid such extra assumptions. 

To fix the terminology, we define the F\"ollmer process as follows:
\begin{definition}[F\"ollmer process]
For $x,y\in\mathbb R^d$, we write $P_x^y$ for the law of the Brownian bridge from $x$ to $y$ induced on $\mathbb W^d$. 
A $d$-dimensional continuous stochastic process $X=(X_t)_{t\in[0,1]}$ is called a \textit{F\"ollmer process} associated with $\mu$ if its law induced on $\mathbb W^d$ is equal to $\int_{\mathbb R^d}P_0^y\mu(dy)$. 
\end{definition}

We can realize a F\"ollmer process associated with $\mu$ as $X_t=t\xi+\beta_t$, $t\in[0,1]$, where $\xi\sim\mu$ and $\beta=(\beta_t)_{t\in[0,1]}$ is a standard Brownian bridge on $[0,1]$ independent of $\xi$. 
The next proposition shows that any F\"ollmer process is realized in this way: 
\begin{proposition}\label{prop:basic}
Let $X=(X_t)_{t\in[0,1]}$ be a F\"ollmer process associated with $\mu$. Define a process $\beta=(\beta_t)_{t\in[0,1]}$ as 
$\beta_t=X_t-tX_1$, $t\in[0,1]$. 
Then $\beta$ is a standard Brownian bridge independent of $X_1$. 
\end{proposition}

\begin{proof}
By definition, the conditional law of $\beta$ given $X_1$ is a standard Brownian bridge in $\mathbb R^d$. 
Since this conditional law does not depend on $X_1$, the unconditional law of $\beta$ is also a standard Brownian bridge and $\beta$ is independent of $X_1$. 
\end{proof}

In the remainder of this paper, $X=(X_t)_{t\in[0,1]}$ always denotes a F\"ollmer process associated with $\mu$ and $\beta=(\beta_t)_{t\in[0,1]}$ is defined as in \cref{prop:basic} unless otherwise is stated. 
\begin{proposition}\label{prop:markov}
$X$ is a Markov process with respect to $\mathbf F^X$. 
Moreover, for $0<r<s\leq1$, the (regular) conditional law of $X_s$ given $X_r=x$ is given by
\ben{\label{eq:trans}
P_{r,s}^X(x,dy)=\frac{\phi_{\frac{r}{s}(s-r)I_d}(x-\frac{r}{s}y)}{p_{r}(x)}P^{X_s}(dy).
}
\end{proposition}

\begin{proof}
Since the law of the time reversal process $\tilde X=(X_{1-t})_{t\in[0,1]}$ is $\int_{\mathbb R^d}P_x^0\mu(dx)$, $\tilde X$ is a Markov process with the same transition function as the Brownian bridge and initial distribution $\mu$. 
Then, by property (iii) in \cite[Section 1.1]{ChWa05}, we have $P(A\mid\mcl F^X_t)=P(A\mid X_t)$ for any $t\in[0,1]$ and $A\in\mcl F^{\tilde X}_t$.  
Since $\{X_s\in B\}\in\mcl F^{\tilde X}_t$ for any $s\in[t,1]$ and Borel set $B\subset\mathbb R^d$, we obtain $P(X_s\in B\mid\mcl F^X_t)=P(X_s\in B\mid X_t)$. 

To prove \eqref{eq:trans}, observe that the conditional law of $\tilde X_{1-r}$ given $\tilde X_{1-r}=y$ has a density $\phi_{\frac{r}{s}(s-r)I_d}\bra{x-\frac{s}{r}y}$, $x\in\mathbb R^d$. Hence, \eqref{eq:trans} follows from Bayes' rule.
\end{proof}

\begin{rmk}
When $H(\mu\mid \gamma_d)<\infty$, the Markov property of $X$ is a consequence of the fact that $X$ is the solution to the dynamical Schr\"odinger problem \eqref{dyn-schr}; see Theorem 3.43 in \cite{FoGa97} or Proposition 5 in \cite{Le14}. 
\end{rmk}

\begin{proposition}\label{prop:f-rc}
The filtration $\mathbf F^X$ is right-continuous. 
\end{proposition}

\begin{proof}
See \cref{proof:f-rc}.
\end{proof}

As announced, $X$ always solves SDE \eqref{eq:follmer} in the weak sense without any assumption on $\mu$:
\begin{theorem}\label{thm:weak}
$X$ is a weak solution to \eqref{eq:follmer}. 
More precisely, we have
\[
P\bra{\int_0^1\abs{\frac{X_t}{t}+\nabla\log p_t(X_t)}dt<\infty}=1
\]
and the process
\[
W^X_t:=X_t-\int_0^t\bra{\frac{X_s}{s}+\nabla\log p_s(X_s)}ds,\quad t\in[0,1],
\]
is a standard $\mathbf F^X$-Wiener process in $\mathbb R^d$. 
\end{theorem}
\begin{proof}
See \cref{proof:weak}.
\end{proof}
\begin{rmk}\label{follmer-work}
When $H(\mu\mid\gamma_d)<\infty$, the claim of \cref{thm:weak} follows from \citet{Fo86}'s results in the following way: Since $H(P^X\mid P^W)=H(\mu\mid\gamma_d)<\infty$ and $X$ is a Markov process, there exist Borel functions $b:[0,1]\times\mathbb R^d\to\mathbb R^d$ and $\wh b:[0,1]\times\mathbb R^d\to\mathbb R^d$ such that both $(X_t-\int_0^tb(s,X_s)ds)_{t\in[0,1]}$ and $(X_{1-t}-\int_0^t\wh b(s,X_{1-s})ds)_{t\in[0,1]}$ are Wiener processes in $\mathbb R^d$ and $\nabla\log p_t(x)=\wh b(1-t,x)+b(t,x)$ for almost all $t\in(0,1)$ and $x\in\mathbb R^d$ by Proposition 2.11 and Theorem 3.10 in \cite{Fo86}. However, since $(X_{1-t})_{t\in[0,1]}$ is a Brownian bridge with initial distribution $\mu$ and terminal value 0, we have $\wh b(t,x)=-x/(1-t)$. Hence we obtain $b(t,x)=x/t+\nabla\log p_t(x)$. 
\end{rmk}

Next we focus on the drift process of SDE \eqref{eq:follmer}:
\[
v_t:=\frac{X_t}{t}+\nabla\log p_t(X_t),\qquad 0<t<1.
\]
The process $(v_t)_{t\in(0,1)}$ is sometimes called the \textit{F\"ollmer drift} associated with $\mu$. 
We begin by discussing its relation to information-theoretic quantities.
\begin{proposition}[\citet{Fo86}, Proposition 2.11]\label{ent-formula}
\ben{\label{eq:ent-formula}
H(\mu\mid \gamma_d)=\frac{1}{2}\int_0^1\E[|v_t|^2]dt.
}
\end{proposition}
Strictly speaking, \citet{Fo86} showed \eqref{eq:ent-formula} when $\mu\ll\gamma_d$ (see also \cite[Theorem 2]{Le13}), but extension to the case $\mu\not\ll\gamma_d$ is straightforward. We give a proof in \cref{sec:ent-formula} for the sake of completeness. 

To state the next result, we need to introduce some notions. 
When a probability distribution $\nu$ on $\mathbb R^d$ has an almost everywhere differentiable density $f$ with respect to $\gamma_d$, we define the \textit{relative Fisher information} of $\nu$ with respect to $\gamma_d$ as 
\[
I(\nu\mid\gamma_d):=\int_{\mathbb R^d}|\nabla\log f|^2d\nu=\int_{\mathbb R^d}\frac{|\nabla f|^2}{f}d\gamma_d.
\]
Also, for $t\in[0,1]$, we denote by $\slepian_t\mu$ the law of a random vector $\sqrt t\xi+\sqrt{1-t}Z$ with $\xi\sim\mu$ and $Z\sim \gamma_d$ being independent. 
$\slepian_t\mu$ is known as the \textit{Slepian interpolation} between $\mu$ and $\gamma_d$. 
Observe that $X_t/\sqrt t\sim\slepian_t\mu$ when $t>0$. 
Also, when $t<1$, $\slepian_t\mu$ has a smooth density with respect to $\gamma_d$ and we denote it by $f_t$.
\begin{proposition}\label{fi-formula}
For every $t\in(0,1)$, we have $v_t=\nabla\log f_{t}(X_t/\sqrt t)/\sqrt t$. Consequently,
\[
\E[|v_t|^2]=\frac{I(\slepian_t\mu\mid \gamma_d)}{t}.
\]
\end{proposition}

\begin{proof}
For any $t\in(0,1)$ and $y\in\mathbb R^d$, we can rewrite $f_t(y)$ as
\[
f_t(y)=\phi_{I_d}(y)^{-1}\int_{\mathbb R^d}\phi_{(1-t)I_d}(y-\sqrt tx)\mu(dx)
=\phi_{I_d}(y)^{-1}\frac{p_t(\sqrt ty)}{t^{d/2}}.
\]
Hence we have $\nabla\log f_t(y)=y+\sqrt t\nabla\log p_t(\sqrt t y)$. Thus $v_t=\nabla\log f_{t}(X_t/\sqrt t)/\sqrt t$.
\end{proof}

Combining Propositions \ref{ent-formula} and \ref{fi-formula} gives a version of \textit{de Bruijn’s identity} (cf.~Proposition 5.2.2 in \cite{BGL14}):
\ben{\label{debruijn}
H(\mu\mid \gamma_d)=\int_0^1\frac{I(\slepian_t\mu\mid\gamma_d)}{2t}dt.
}

To obtain another useful representation of the F\"ollmer drift, we introduce a family of probability distributions on $\mathbb R^d$. 
For $t\in(0,1)$ and $x\in\mathbb R^d$, define a probability distribution $\mu_{t,x}$ on $\mathbb R^d$ as
\[
\mu_{t,x}(dy)=\frac{e^{\frac{|y|^2}{2}-\frac{|y-x|^2}{2(1-t)}}}{\int_{\mathbb R^d}e^{\frac{|z|^2}{2}-\frac{|z-x|^2}{2(1-t)}}\mu(dz)}\mu(dy).
\]
Since $y\mapsto|y|^pe^{\frac{|y|^2}{2}-\frac{|y-x|^2}{2(1-t)}}$ is bounded for all $p>0$, $\mu_{t,x}$ has all moments. 
We particularly set
\[
m(t,x):=\int_{\mathbb R^d}y\mu_{t,x}(dy).
\]
A straightforward computation shows
\ben{\label{vt-formula}
v_t=\frac{m(t,X_t)-X_t}{1-t}\quad\text{for all }0<t<1.
}
\begin{proposition}[\citet{ElLe18}, Eq.(21)]\label{lem:sl}
For any $t\in(0,1)$ and $x\in\mathbb R^d$, $\mu_{t,x}$ is the (regular) conditional law of $X_1$ given $X_t=x$. 
In particular, if $\E[|X_1|]<\infty$,
\ben{\label{mt-formula}
m(t,X_t)=\E[X_1\mid X_t].
}  
\end{proposition}

\begin{proof}
Observe that
\ba{
\frac{d\mu_{t,x}}{d\mu}(y)=\frac{e^{-\frac{|ty-x|^2}{2t(1-t)}}}{\int_{\mathbb R^d}e^{-\frac{|tz-x|^2}{2t(1-t)}}\mu(dz)}
=\frac{\phi_{t(1-t)I_d}(x-ty)}{p_t(x)}.
}
Therefore, the asserted claim is a consequence of \cref{prop:markov}.
\end{proof}

When $\E[|X_1|]<\infty$, we have $\E[X_1\mid X_t]\to\E[X_1]$ a.s.~as $t\downarrow0$ by the Markov property of $X$, \cite[Corollary II.2.4]{ReYo99} and \cref{prop:f-rc}. 
Hence, $v_t\to\E[X_1]$ as $t\downarrow0$ a.s.~by \eqref{vt-formula} and \eqref{mt-formula}. 
For this reason we set $v_0:=\E[X_1]$ whenever $\E[|X_1|]<\infty$. 
\begin{rmk}[Stochastic localization]
The measure-valued stochastic process $(\mu_{t,X_t})_{t\in(0,1)}$ is closely related to \citet{El13}'s stochastic localization; see \cite[Section 4]{KlPu23} for details. 
\citet{BDDD24} successfully utilize stochastic localization to obtain a refined error bound for discretization of the reverse SDE in SGM. 
\end{rmk}

The following three propositions are useful consequences of \eqref{vt-formula} and \cref{lem:sl}. 

\begin{proposition}[\citet{EMZ20}, Lemma 11]
Assume $\E[|X_1|^2]<\infty$. Then for every $t\in[0,1)$,
\ben{\label{eq:v2}
\E[|v_t|^2]=\frac{d-\E[\trace(\Gamma_t)]}{1-t}
+\E[|X_1|^2]-d,
}
where
\ben{\label{def:gamma}
\Gamma_t:=\frac{\covariance[X_1\mid X_t]}{1-t}.
}
\end{proposition}

\begin{proof}
Using \eqref{vt-formula} and \eqref{mt-formula}, we deduce
\ba{
\E[|v_t|^{2}]
&=\frac{\E[|\E[X_1\mid X_t]-X_t|^{2}]}{(1-t)^2}
=\frac{\E[|\E[X_1\mid X_t]|^{2}]-2\E[X_1\cdot X_t]+\E[|X_t|^{2}]}{(1-t)^2}.
}
By \cref{prop:basic}, we have $\E[X_1\cdot X_t]=t\E[|X_1|^{2}]$ and $\E[|X_t|^{2}]=t^2\E[|X_1|^{2}]+dt(1-t)$. Consequently, 
\ba{
\E[|v_t|^{2}]
&=\frac{-\E[\trace(\covariance[X_1\mid X_t])]+(1-t)^2\E[|X_1|^{2}]+dt(1-t)}{(1-t)^2}
=\frac{d-\E[\trace(\Gamma_t)]}{1-t}
+\E[|X_1|^{2}]-d.
}
\end{proof}

\begin{proposition}[\citet{ELS20}, Lemma 13]
Assume $\E[|X_1|^2]<\infty$. Then, $\E[\int_0^1\trace(\Gamma_t^2)dt]<\infty$ and
\[
X_1=\E[X_1]+\int_0^1\Gamma_tdW^X_t.
\]
\end{proposition}

\begin{proof}
Fix $\eps\in(0,1/2)$ arbitrarily. It is straightforward to verify that the map $[\eps,1-\eps]\times\mathbb R^d\ni(t,x)\mapsto m(t,x)\in\mathbb R^d$ is of class $C^{1,2}$. 
Hence, Ito's formula gives
\ben{\label{eq:ito}
m(t,X_t)=m(\eps,X_\eps)+\int_\eps^{t}\frac{\partial m}{\partial s}(s,X_s)ds
+\sum_{j=1}^d\int_\eps^{t}\frac{\partial m}{\partial x_j}(s,X_s)dX^j_s
+\frac{1}{2}\sum_{j=1}^d\int_\eps^{t}\frac{\partial^2m}{\partial x_j^2}(s,X_s)ds 
}
for all $t\in[\eps,1-\eps]$, where $X^j_s$ is the $j$-th component of $X_s$. 
Meanwhile, since $m(t,X_t)=\E[X_1\mid\mcl F^X_t]$ for any $t\in(0,1)$ by \eqref{mt-formula} and the Markov property of $X$, $(m(t,X_t))_{t\in[\eps,1-\eps]}$ is a continuous martingale with respect to $(\mcl F^X_t)_{t\in[\eps,1-\eps]}$. Hence, \eqref{eq:ito} and \cite[Proposition IV.1.2]{ReYo99} give
\ben{\label{mar-rep}
\E[X_1\mid\mcl F^X_t]=\E[X_1\mid\mcl F^X_\eps]
+\sum_{j=1}^d\int_\eps^{t}\frac{\partial m}{\partial x_j}(s,X_s)dW^{X,j}_s,
}
where $W^{X,j}_s$ is the $j$-th component of $W^X_s$.  
A straightforward computation shows
\[
\frac{\partial m_k}{\partial x_j}(s,x)=\frac{1}{1-s}\int_{\mathbb R^d}(y_jy_k-m_j(s,x)m_k(s,x))\mu_{s,x}(dy),
\]
where $m_k(s,x)$ denote the $k$-th component of $m(s,x)$. 
Combining this with \eqref{mar-rep} and \cref{lem:sl} gives
\ben{\label{ito-rep}
\E[X_1\mid\mcl F^X_{1-\eps}]=\E[X_1\mid\mcl F^X_\eps]
+\int_\eps^{1-\eps}\Gamma_tdW^X_t.
}
Using this expression, we obtain 
\ba{
\E\sbra{\int_\eps^{1-\eps}\trace(\Gamma_t^2)dt}
=\E[|\E[X_1\mid\mcl F^X_{1-\eps}]-\E[X_1\mid\mcl F^X_\eps]|^2]\leq\E[|X_1|^2],
}
where the first equality is by \cite[Proposition IV.1.23]{ReYo99} and the last inequality is by Jensen's inequality. 
Therefore, Fatou's lemma gives $\E[\int_0^1\trace(\Gamma_t^2)dt]\leq\E[|X_1|^2]<\infty$. 
This particularly implies that the stochastic integral $\int_0^1\Gamma_tdW^X_t$ is well-defined. 
Then, since $\E[X_1\mid\mcl F^X_{1-\eps}]\to X_1$ and $\E[X_1\mid\mcl F^X_\eps]\to\E[X_1]$ a.s.~as $\eps\downarrow0$ by \cite[Corollary II.2.4]{ReYo99} and \cref{prop:f-rc}, we obtain the desired result by letting $\eps\downarrow0$ in \eqref{ito-rep}. 
\end{proof}

\begin{proposition}[\cite{ElLe18}, Fact 2.2]\label{prop:vt-mar}
Assume $\E[|X_1|]<\infty$. Then $(v_t)_{t\in[0,1)}$ is an $\mathbf F^X$-martingale in $\mathbb R^d$. 
\end{proposition}

\begin{proof}
See \cref{proof:vt-mar}. 
\end{proof}

We conclude this section by a simple criterion for the finiteness of moments of the F\"ollmer drift. 
\begin{proposition}\label{prop:vt-mom}
For any $r>0$, the following three conditions are equivalent: 
(i) $\E[|X_1|^r]<\infty$; 
(ii) $\E[|v_t|^r]<\infty$ for all $t\in(0,1)$; 
(iii) $\E[|v_t|^r]<\infty$ for some $t\in(0,1)$.
\end{proposition}

\begin{proof}
See \cref{proof:vt-mom}. 
\end{proof}

\section{Connection to the denoising diffusion probabilistic model}\label{sec:main}

We begin by introducing the SDE formulation of the DDPM formally. 
Let $Y=(Y_t)_{t\geq0}$ be a solution to the following SDE:
\ben{\label{eq:ou}
dY_t=-Y_tdt+\sqrt2dW_t,\qquad Y_0\sim\mu.
}
For $T>0$, let $\ol Y=(\ol Y_t)_{t\in[0,T]}$ be the time reversal of $(Y_t)_{t\in[0,T]}$, i.e.~$\ol Y_t=Y_{T-t}$ for $t\in[0,T]$. Then, by \cite[Theorem 2.1]{HaPa86}, \cite[Proposition VII.2.4]{ReYo99} and L\'evy's characterization, there exists a standard Wiener process $\ol W=(\ol W_t)_{t\in[0,T)}$ in $\mathbb R^d$ such that 
\ben{\label{eq:score-diffusion}
\ol Y_t=\ol Y_0+\int_0^t\bra{\ol Y_s+2\nabla\log q_{T-s}(\ol Y_s)}ds+\sqrt{2}\ol{W}_t
}
for all $t\in[0,T)$, where $q_t$ is the density of $Y_t$ for every $t>0$. 

As already mentioned in the introduction, the F\"ollmer process can be seen as an ``augmented'' time-compressed version of the reverse SDE \eqref{eq:score-diffusion}. 
This fact is essentially known in the literature, as the F\"ollmer process is an (augmented) time-compressed version of Eldan's stochastic localization process (see \cite[Section 4.2]{KlPu23}) and the process $\ol Y$ in \eqref{eq:score-diffusion} is a change of variables of the latter (see \cite[Section 1.3]{Mo23} or \cite[Theorem 3]{STZ25}). The following proposition formulates this fact more explicitly: 
\begin{proposition}\label{thm:rou}
For $T>0$, define a process $\ol Y=(\ol Y_t)_{t\in[0,T]}$ as $\ol Y_t=e^{T-t}X_{e^{2(t-T)}}$, $t\in[0,T]$.
Then $\ol Y_0\sim\slepian_{e^{-2T}}\mu$ and the process
\[
\ol W_t=\frac{1}{\sqrt 2}\bra{\ol Y_t-\ol Y_0-\int_0^t\bra{\ol Y_s+2\nabla\log q_{T-s}(\ol Y_s)}ds}\quad (t\in[0,T])
\]
is a standard $\mathbf F^{\ol Y}$-Wiener process in $\mathbb R^d$. 
\end{proposition}

\begin{proof}
The first claim is immediate from the definition of $\ol Y$. 
To prove the second claim, define a process $V=(V_t)_{t\geq0}$ as $V_t=W^X_{e^{2(t-T)}}$ for $t\in[0,T]$. By Proposition V.1.5 in \cite{ReYo99}, $V$ is an $(\mcl F^X_{e^{2(t-T)}})$-martingale in $\mathbb R^d$ and $[V,V]_t=e^{2(t-T)}I_d$ for every $t\in[0,T]$. 
Hence the process
\[
\wt W_t=\frac{1}{\sqrt2}\int_0^te^{T-s}dV_s\quad(t\in[0,T])
\]
is an $(\mcl F^X_{e^{2(t-T)}})$-Wiener process in $\mathbb R^d$ by L\'evy's characterization. 
Since $\mcl F^X_{e^{2(t-T)}}=\mcl F^{\ol Y}_t$ for all $t\in[0,T]$, we complete the proof once we show that $\ol W_t=\wt W_t$ for every $t\in[0,T]$. 

Observe that $q_t(y)=p_{e^{-2t}}(e^{-t}y)$ for $t>0$ and $y\in\mathbb R^d$. Hence
\ba{
\int_{e^{-2T}}^{e^{2(t-T)}}\bra{\frac{X_s}{s}+\nabla\log p_s(X_s)}ds
&=2\int_{0}^{t}\bra{\frac{X_{e^{2(u-T)}}}{e^{2(u-T)}}+\nabla\log p_{e^{2(u-T)}}(X_{e^{2(u-T)}})}e^{2(u-T)}du\\
&=2\int_{0}^{t}\bra{\ol Y_{u}+\nabla\log q_{T-u}(\ol Y_{u})}e^{u-T}du.
}
Therefore, integration by parts gives
\ba{
\ol Y_t-\ol Y_0
&=2\int_{0}^{t}\bra{\ol Y_{u}+\nabla\log q_{T-u}(\ol Y_{u})}du+\sqrt2\wt W_t-\int_0^t\ol Y_udu\\
&=\int_{0}^{t}\bra{\ol Y_{u}+2\nabla\log q_{T-u}(\ol Y_{u})}du+\sqrt2\wt W_t.
}
Hence $\ol W_t=\wt W_t$. 
\end{proof}

\cref{thm:rou} suggests that discretization of the F\"ollmer process could be used to analyze sampling errors of the DDPM. In this section, we implement this analysis while surveying relevant error bounds for the DDPM sampler. 

\subsection{Discretized F\"ollmer process and the DDPM sampler}

Given i.i.d.~data generated from $\mu$, we aim to construct a generator that simulates a random vector whose distribution is close to $\mu$. 
Suppose that we have a good estimator $\score(x,t)$ for $\nabla\log p_t(x)$. Then, by simulating a solution $\wh X=(\wh X_t)_{t\in[0,1]}$ to the SDE
\ben{\label{eq:score-follmer}
d\wh X_t=\bra{\frac{\wh X_t}{t}+\score(\wh X_t,t)}dt+dW_t\quad(0\leq t\leq1),\qquad \wh X_0=0,
}
we expect that the law of $\wh X_1$ would be close to $\mu$. 
In practice, we need to discretize the SDE \eqref{eq:score-follmer} for the simulation. 
Below we discuss error bounds for such discretization. 

Let $0<t_0<t_1<\cdots<t_N\leq1$ be time points for discretization and set $\delta:=1-t_N$. 
We will perform early stopping so that $\delta>0$ when $\mu$ is not smooth enough because the score $\nabla\log p_t(x)$ may blow up as $t\uparrow1$ in such a situation. 
Nevertheless, $\delta$ should be sufficiently small so that the error induced by early stopping is small. 
Also, analogously to \cite{BDDD24}, we introduce the following conditions:
\begin{enumerate}[label={\normalfont[A\arabic*]}]

\item\label{ass:sampling}
There exists a constant $\kappa>0$ such that
\ben{\label{eq:sampling}
h_i:=t_{i+1}-t_i\leq\kappa(1-t_{i+1})\quad(i=0,1,\dots,N-1).
}

\item\label{ass:score}
The score estimator $\score:\mathbb R^d\times(0,1)\to\mathbb R^d$ is measurable and satisfies
\ben{\label{score-bound}
\sum_{i=0}^{N-1}(t_{i+1}-t_i)\E\sbra{\abs{\nabla\log p_{t_i}(X_{t_i})-\score(X_{t_i},t_i)}^2}\leq\eps_\text{score}^2
}
for some $\eps_\text{score}\geq0$. 

\end{enumerate}
We remark that \eqref{eq:sampling} and \eqref{score-bound} are implied by the corresponding conditions in terms of the process $\ol Y_t=e^{T-t}X_{e^{2(t-T)}}$, $t\in[0,T]$, as shown in the following lemma. 
Observe that the sampling times corresponding to $\ol Y$ are given by $\tau_i=T+\frac{1}{2}\log t_i$ for $i=0,1,\dots,N$. 

\begin{lemma}\label{lem:sampling}
Assume
\ben{\label{eq:tau-samp}
\tau_{i+1}-\tau_i\leq\kappa\min\{1,T-\tau_{i+1}\}\quad (i=0,1,\dots,N-1)
}
for some constant $\kappa>0$. Then we have \eqref{eq:sampling}. 
Further, suppose that a measurable function $\wt{\mf s}:\mathbb R^d\times(0,T)\to\mathbb R^d$ satisfies
\ben{\label{eq:tau-score}
\sum_{i=0}^{N-1}(\tau_{i+1}-\tau_i)\E\sbra{\abs{\nabla\log q_{T-\tau_i}(\ol Y_{\tau_i})-\wt{\mf s}(\ol Y_{\tau_i},T-\tau_i)}^2}\leq\eps^2
}
for some $\eps\geq0$. 
Then we have \eqref{score-bound} with $\score(x,t)=\frac{1}{\sqrt t}\wt{\mf s}\bra{\frac{x}{\sqrt t},-\frac{1}{2}\log t}$ and $\eps_\mathrm{score}=\sqrt 2e^\kappa\eps$.
\end{lemma}

\begin{proof}
Observe that
\[
\tau_{i+1}-\tau_i=\frac{1}{2}\log(t_{i+1}/t_i)
\quad\text{and}\quad
T-\tau_{i+1}=\frac{1}{2}\log(1/t_{i+1}).
\]
Hence,
\ben{\label{eq:dt-dtau}
t_{i+1}-t_i=t_{i+1}(1-t_i/t_{i+1})\leq 2t_{i+1}(\tau_{i+1}-\tau_i) 
\leq \kappa t_{i+1}\log(1/t_{i+1})\leq\kappa(1-t_{i+1}),
}
where the first and third inequalities follow by $\log x\leq x-1$ for any $x>0$ and the second by \eqref{eq:tau-samp}. 
Thus, we have \eqref{eq:sampling}. 
Also, since
\ba{
&\sum_{i=0}^{N-1}(t_{i+1}-t_i)\E\sbra{\abs{\nabla\log p_{t_i}(X_{t_i})-\score(X_{t_i},t_i)}^2}\\
&=\sum_{i=0}^{N-1}(1-t_i/t_{i+1})\frac{t_{i+1}}{t_i}\E\sbra{\abs{\nabla\log q_{T-\tau_i}(Y_{T-\tau_i})-\wt{\mf s}(Y_{T-\tau_i},T-\tau_i)}^2}\\
&\leq2e^{2\kappa}\sum_{i=0}^{N-1}(\tau_{i+1}-\tau_i)\E\sbra{\abs{\nabla\log q_{T-\tau_i}(Y_{T-\tau_i})-\wt{\mf s}(Y_{T-\tau_i},T-\tau_i)}^2},
}
we have \eqref{score-bound} with $\eps_\text{score}=\sqrt 2e^\kappa\eps$. 
\end{proof}
We also note that $T-\tau_N$ has the same order as $\delta$. In fact, the inequality $1-x\leq\log(1/x)\leq(1-x)/x$ gives  
\[
\frac{\delta}{2}\leq T-\tau_N\leq\frac{\delta}{2(1-\delta)}.
\]
Consequently, we can construct an example of $(t_i)_{i=0}^N$ satisfying \ref{ass:sampling} from the two-stage arithmetic-geometric grid commonly used in the literature. Formally, given constants $T>\delta>0$, let
\[
\tau_i=
\begin{cases}
\frac{2(T-1)i}{N} & \text{if }0\leq i\leq N/2,\\
T-\delta^{2(i/N-1/2)}& \text{if }N/2< i\leq N.
\end{cases}
\] 
Then $t_i=e^{2(\tau_i-T)}$ $(i=0,1,\dots,N)$ satisfy \ref{ass:sampling} with $\kappa\asymp\frac{T+\log(1/\delta)}{N}$ (cf.~the proof of \cite[Corollary 1]{BDDD24}). 

Note that we need the condition $\max_{0\leq i\leq N-1}(\tau_{i+1}-\tau_i)=o(1)$ to ensure that the discretization error for $\ol Y$ is small. 
In view of \eqref{eq:dt-dtau}, this condition implies $\max_i(1-t_i/t_{i+1})=o(1)$, which is strictly stronger than the condition $\max_i(t_{i+1}-t_i)=o(1)$ needed to control the discretization error for $X$. 
Below we give several examples satisfying \ref{ass:sampling} for which the condition $\max_i(1-t_i/t_{i+1})=o(1)$ fails or imposes an additional constraint (see also \cref{rmk:fi}).


\begin{example}[Cosine schedule]\label{ex:cos}
Given a constant $s\geq0$, we set 
\[
t_i:=\sin^2(i\eta_N)\quad\text{for }i=1,\dots,N,
\quad\text{where }\eta_N:=\frac{\pi}{2N(1+s)}.
\]
Also, let $t_0\in(0,t_1)$. 
This corresponds to the cosine schedule proposed in \citet{NiDh21}. 
Note that we use the sine function rather than the cosine function because our indexing convention is the reverse of that in \cite{NiDh21}. 
In this case, $\delta=\sin^2(\frac{\pi s}{2(1+s)})\leq\frac{\pi^2 s^2}{4}$. 
Moreover, for $i=0,1,\dots,N-1$,
\ba{
\frac{h_i}{1-t_{i}}\leq\frac{\sin((2i+1)\eta_N)\sin(\eta_N)}{\cos^2(i\eta_N)}
\leq\frac{2\sin(\eta_N)}{\cos(i\eta_N)}
\leq\frac{2\sin(\eta_N)}{\sin(\frac{\pi s}{2(1+s)})}
\leq\frac{\pi}{Ns}.
}
Therefore, provided that $Ns>\pi$, we have \eqref{eq:sampling} with
\[
\kappa=\frac{\pi}{Ns-\pi}.
\] 
Meanwhile, since $\eta_N\to0$ as $N\to\infty$, we have
\[
\frac{t_1}{t_2}
=\frac{\sin^2(\eta_N)}{\sin^2(2\eta_N)}
=\frac{1}{4\cos^2(\eta_N)}
\to\frac{1}{4}\quad\text{as }N\to\infty.
\]
Therefore, $\max_i(1-t_i/t_{i+1})\not\to0$ as $N\to\infty$.
\end{example}

\begin{example}[Geometric-type schedule]\label{ex:geom}
Let us consider the case where \eqref{eq:sampling} holds with equality. In this case, 
\ben{\label{eq:geom}
t_i=1-(1+\kappa)^{-i}(1-t_0)\qquad(i=0,1,\dots,N).
}
Provided that $\kappa<1/2$, we have $\delta\leq e^{-N\kappa/2}$. 
Therefore, choosing
\[
\kappa=\frac{2c\log N}{N}
\]
for some constant $c>0$, we have $\delta\leq N^{-c}$ for sufficiently large $N$.  
Meanwhile, a straightforward computation gives
\[
1-\frac{t_0}{t_1}=\kappa\frac{1-t_0}{\kappa+t_0}.
\]
Therefore, we need $t_0\gg\kappa$ to have $\max_i(1-t_i/t_{i+1})=o(1)$. 
\end{example}

We note that \ref{ass:sampling} requires $\delta>0$. Therefore, the linear schedule of \cite{HJA20} does not satisfy \ref{ass:sampling}.

%
\begin{rmk}[Score estimation]\label{rmk:score}
The score estimator $\score$ is typically constructed by minimizing the left hand side of \eqref{score-bound} (or \eqref{eq:tau-score}) with respect to $\score$ over a class of neural networks. 
It is known that for all $0<t<1$,
\ban{
&\E\sbra{\abs{\nabla\log p_{t}(X_{t})-\score(X_{t},t)}^2}\notag\\
&=\E\sbra{\abs{\frac{-Z}{\sqrt{t(1-t)}}-\score(\xi_t,t)}^2}
-\E\sbra{\abs{\frac{-Z}{\sqrt{t(1-t)}}-\nabla\log p_{t}(\xi_{t})}^2},\label{eq:score-match}
}
where $\xi_t=t\xi+\sqrt{t(1-t)}Z$ with $\xi\sim\mu$ and $Z\sim \gamma_d$ being independent; see \cref{proof:score-match} for a proof.  
Since the second term of the right hand side on \eqref{eq:score-match} is independent of $\score$, the above minimization problem is equivalent to minimizing 
\[
\sum_{i=0}^{N-1}(t_{i+1}-t_i)\E\sbra{\abs{\frac{-Z}{\sqrt{t_i(1-t_i)}}-\score(\xi_{t_i},t_i)}^2}
\]
with respect to $\score$. 
In this formulation, both the ``features'' $\xi_{t_i}$ and the ``predicted variables'' $-Z/\sqrt{t_i(1-t_i)}$ have different scales across $i$, causing numerical instability. 
Instead, \citet{HJA20} have proposed reparametrizing the score estimator to predict $Z$. In our setting, this leads to minimizing the following objective function with respect to $\hat\epsilon:\mathbb R^d\times(0,1)\to\mathbb R^d$:
\ben{\label{eq:noise-predictor}
\sum_{i=0}^{N-1}\frac{t_{i+1}-t_i}{t_i(1-t_i)}\E\sbra{\abs{Z-\hat\epsilon(\sqrt{t_i}\xi+\sqrt{t_i(1-t_i)}Z,t_i)}^2}.
}
Then, the score estimator is obtained as
\[
\score(x,t)=-\frac{1}{\sqrt{t(1-t)}}\hat\epsilon\bra{\frac{x}{\sqrt t},\,t}.
\] 
In view of the correspondence between the discretized F\"ollmer process and the DDPM sampler explained below, the objective function \eqref{eq:noise-predictor} is the same as Eq.(12) of \cite{HJA20}. 
Further, since $\mu$ is unknown, one usually replaces $\mu$ by the empirical distribution of observation data. 
The estimation error induced by this procedure can be studied separately and we do not pursue this topic here. We refer to \cite{FYWC24,OAS23} for recent developments of this topic. See also the discussion before \cref{coro:score}.
\end{rmk}

To discretize the SDE \eqref{eq:score-follmer}, we first focus on the Euler--Maruyama scheme. 
Namely, we construct a discretized process $(\euler_{t_i})_{i=0}^N$ as
\be{
\begin{cases}
\euler_{t_0}=\sqrt{t_0}Z_0,\\
\euler_{t_{i+1}}=\euler_{t_{i}}+\bra{\frac{\euler_{t_{i}}}{t_{i}}+\score(\euler_{t_{i}},t_{i})}h_i+\sqrt{h_i}Z_i\quad(i=0,1,\dots,N-1),
\end{cases}
}
where $Z_i\overset{i.i.d.}{\sim}N(0,I_d)$. 
As anticipated from \cref{thm:rou}, this algorithm can be interpreted as a special case of the DDPM sampler. 
To formulate this precisely, we introduce the DDPM sampler in a general form, following \cite{LiYa24}. 
Given constants $\alpha_i\in(0,1)$ $(i=0,1,\dots,N)$, suppose that we have an estimator $\scored_i(x)$ for the score function of $\slepian_{\ol\alpha_{i}}\mu$ for each $i=0,1,\dots,N-1$, where $\ol\alpha_i:=\prod_{j=i}^{N}\alpha_j$. 
Then, the DDPM sampler sequentially generates random vectors $\ddpm_i$ $(i=0,1,\dots,N)$ according to the following update rule:
\ben{\label{eq:ddpm}
\ddpm_0\sim N(0,I_d),\qquad
\ddpm_{i+1}=\frac{1}{\sqrt{\alpha_{i}}}\bra{\ddpm_{i}+\eta_i\scored_i(\ddpm_i)+\sigma_iZ_i}\quad(i=0,1,\dots,N-1),
}
where $\eta_i,\sigma_i>0$ are hyper-parameters chosen so that the law of $\ddpm_i$ is close to $\slepian_{\ol\alpha_{i}}\mu$ for each $i$. 
Noting that $X_t/\sqrt t\sim\slepian_t\mu$ for each $t\in(0,1]$, one can easily verify that $\ddpm_i=\euler_{t_i}/\sqrt{t_i}$ $(i=0,1,\dots,N)$ correspond to samples from the DDPM sampler with $\scored_i(x)=\sqrt{t_i}\score(\sqrt{t_i}x,t_i)$, $\alpha_i=t_i/t_{i+1}$, $\eta_i=1-\alpha_i$, $\sigma_i=\sqrt{\alpha_i(1-\alpha_i)}$ for $i=0,1,\dots N-1$ and $\alpha_N=t_N$. 
Interestingly, this hyper-parameter setting is the \textit{exactly} same as the standard choice used in the discrete-time model, originally proposed in \cite{HJA20} (see also Section 2.2 of \cite{SSKKEP21}). 
By contrast, the \textit{exponential integrator scheme}, a popular scheme to discretize the SDE \eqref{eq:score-diffusion}, leads to the following hyper-parameter setting: $\alpha_i=e^{-2(\tau_{i+1}-\tau_i)}=t_i/t_{i+1}$, $\eta_i=2(1-\sqrt{\alpha_i})$, $\sigma_i=\sqrt{1-\alpha_i}$ for $i=0,1,\dots N-1$ and $\alpha_N=e^{-2(T-\tau_N)}=t_N$. 
See e.g.~Eq.(5) in \cite{DeB22} for the explicit construction of the exponential integrator scheme. 
This is only approximately equivalent to the standard choice above. 
More importantly, the latter approximation is justified only when $\max_{i=0,1,\dots,N-1}(1-\alpha_i)=o(1)$ as we need $\max_{i=0,1,\dots,N-1}(\tau_{i+1}-\tau_i)=o(1)$. This rules out some standard choices of $\alpha_i$ such as the cosine schedule of \cite{NiDh21}, as already seen in \cref{ex:cos}. 

%

%

We proceed to giving concrete bounds for discretization errors. 
The first result is a counterpart of Theorem 2 in \citet{BDDD24}: 
\begin{theorem}\label{thm:em}
Assume \ref{ass:sampling}--\ref{ass:score} and $t_0\leq1/2$. Then
\[
H(P^{X_{t_N}}\mid P^{\euler_{t_N}})\leq 
\eps_\mathrm{score}^2
+\kappa d\log\bra{\frac{1-t_0}{\delta}}
+2\kappa\E[|X_1|^2]
+\frac{t_0}{2}\bra{d+\E[|X_1|^2]}.
\]
\end{theorem}
The proof of \cref{thm:em} is given in \cref{proof:em}. 
We note that the bound in \cref{thm:em} is applicable to $\ddpm_N=\euler_{t_N}/\sqrt{t_N}$ because $H(\slepian_{t_N}\mu\mid P^{\ddpm_N})=H(P^{X_{t_N}}\mid P^{\euler_{t_N}})$ by \cref{lem:dpi}. 
In particular, \cref{lem:sampling} gives the following corollary:

\begin{corollary}\label{coro:em}
Suppose that $\tau_i=T+\frac{1}{2}\log t_i$ $(i=0,1,\dots,N)$ with $T=-\frac{1}{2}\log t_0$ satisfy \eqref{eq:tau-samp} and \eqref{eq:tau-score}. 
Define $\score(x,t)=\frac{1}{\sqrt t}\wt{\mf s}\bra{\frac{x}{\sqrt t},-\frac{1}{2}\log t}$. 
Then $\ddpm_N=\euler_{t_N}/\sqrt{t_N}$ satisfies
\[
H(\slepian_{t_N}\mu\mid P^{\ddpm_N})\leq 2e^{2\kappa}\eps^2
+\kappa d\log\bra{\frac{1-t_0}{\delta}}+2\kappa\E[|X_1|^2] 
+\frac{t_0}{2}\bra{d+\E[|X_1|^2]}.
\]
\end{corollary}

\begin{rmk}[Comparison to \citet{BDDD24}]\label{rmk:bddd}
When $\ddpm_i$ are obtained via discretizing $\ol Y_t=e^{T-t}X_{e^{2(t-T)}}$ by the exponential integrator scheme with sampling times $\tau_i=T+\frac{1}{2}\log t_i$, Theorem 2 in \cite{BDDD24} gives the following upper bound for $H(\slepian_{t_N}\mu\mid P^{\ddpm_N})$ up to a universal constant under \eqref{eq:tau-samp} and \eqref{eq:tau-score}:
\ben{\label{eq:bddd}
\eps^2
+\kappa^2dN+\kappa dT+\kappa\E[|X_1|^2] 
+\bra{d+\E[|X_1|^2]}e^{-2T}.
}
Recall that $T$ is related to $t_0$ as $t_0=e^{-2T}$. 
Also, since
\[
(\kappa N+1)\delta\geq \sum_{i=0}^{N-1}(t_{i+1}-t_{i})+1-t_N=1-t_0,
\]
we have $\log((1-t_0)/\delta)\leq(1-t_0)/\delta-1\leq\kappa N$. 
Therefore, \cref{coro:em} gives the following bound for $H(\slepian_{t_N}\mu\mid P^{\ddpm_N})$ up to a universal constant under the same assumption as above with assuming $\kappa\leq1$:
\[
\eps^2
+\kappa^2dN+\kappa\E[|X_1|^2] 
+\bra{d+\E[|X_1|^2]}e^{-2T}.
\]
Compared to \eqref{eq:bddd}, our bound does not contain the term $\kappa dT$. 
In particular, except for implicit dependence of the score estimation error on $T$ (see \cref{rmk:score}), a large value of $T$ does not hurt our bound. 
In view of the correspondence to the DDPM sampler discussed above, this is natural because the parameters in \eqref{eq:ddpm} is determined via ratios $t_i/t_{i+1}$ which do not depend on $T$. 
Note that the appearance of $e^{-2T}$ in the bounds is due to the initialization error. 
\end{rmk}

\begin{rmk}
Recently, \citet{JZL25} have established a bound for $H(\slepian_{t_N}\mu\mid P^{\ddpm_N})$ of order $O\bra{(\kappa d)^2}$ with respect to $\kappa$ from \eqref{eq:tau-samp} and $d$ up to some logarithmic factor (see also \cite{JaZh26}). 
Their proof approach is fundamentally different from ours: They do not view the DDPM sampler as a discretized SDE but construct the corresponding ``true'' denoising process via a noise-injected probability flow ordinary differential equation of \cite{SSKKEP21} (or the Kim--Milman map in terms of \cite{KlPu23}). It is unclear whether our approach can give a similar bound.
\end{rmk}

When the training sample size is $n$ and $\mu$ has a compactly supported density that is positive on the support and belongs to a H\"older ball of regularity $\gamma>0$, we expect that $\eps_\text{score}^2$ is of order $n^{-\frac{2\gamma}{d+2\gamma}}$ up to a log factor by \cite[Eq.(4)]{OAS23}. 
Moreover, in this case, we also have $\tv(\slepian_{t_N}\mu,\,\mu)=O(n^{-\frac{\gamma}{d+2\gamma}})$ if $\delta=n^{-c}$ for sufficiently large constant $c$ by \cite[Theorem D.2]{OAS23}, where $\tv$ denotes the total variation distance. 
Under these circumstances, the following corollary determines the necessary number of sampling steps to achieve an $O(n^{-\frac{2\gamma}{d+2\gamma}})$ error in the total variation distance up to a log factor.  
\begin{corollary}\label{coro:score}
Suppose that $\E[|X_1|^2]<\infty$ and \ref{ass:score} with $\eps_\text{score}=O(n^{-\frac{\gamma}{d+2\gamma}})$ as $n\to\infty$ for some constant $\gamma>0$. 
Let $N\asymp n^{\frac{2\gamma}{d+2\gamma}}$ and define $(t_i)_{i=0}^N$ by \eqref{eq:geom} with $t_0=O(n^{-\frac{\gamma}{d+2\gamma}})$ and $\kappa=2cN^{-1}\log N$ for some constant $c>0$. Further, suppose $\tv(\slepian_{t_N}\mu,\mu)=O(n^{-\frac{\gamma}{d+2\gamma}})$. Then, for $\ddpm_N=\euler_{t_N}/\sqrt{t_N}$,
\[
\tv(P^{\ddpm_N},\,\mu)=O\bra{n^{-\frac{\gamma}{d+2\gamma}}\log^2 n}.
\] 
\end{corollary}
This corollary is an immediate consequence of \cref{thm:em} and Pinsker's inequality. 
While we can formulate results similar to this corollary for the subsequent results, we omit the details to save space. 

When $\mu$ has a finite relative Fisher information with respect to $\gamma_d$, we can bound the discretization error without early stopping. 
The next result is a counterpart of Theorem 2.1 in \citet{CDGS24}:
\begin{theorem}\label{thm:fi}
Assume \ref{ass:score}, $t_0\leq1/2$, $t_N=1$ and that $\mu$ has an almost everywhere differentiable density. 
Then, with $\bar h:=\max_{i=0,1,\dots,N-1}h_i$,
\[
H(\mu\mid P^{\euler_{1}})\leq t_0H(\mu\mid\gamma_d)+\eps_{\mathrm{score}}^2+I(\mu\mid\gamma_d)\bar h.
\]
\end{theorem}
We show this result in \cref{proof:fi}. 

\begin{rmk}[Comparison to \citet{CDGS24}]\label{rmk:fi}
When $\ddpm_i$ are obtained as in \cref{rmk:bddd} and $\tau_{i}$ are equidistant, 
Theorem 2.1 in \cite{CDGS24} gives the following upper bound for $H(\slepian_{t_N}\mu\mid P^{\ddpm_N})$ up to a universal constant under \eqref{eq:tau-score}:
\[
e^{-2T}H(\mu\mid\gamma_d)+\eps^2+I(\mu\mid\gamma_d)\frac{T}{N}.
\]
Since the assumptions of \cref{thm:fi} are satisfied with $\bar h=T/N$ in this case, our result recovers this bound. 
In addition, if we take $t_i$ to be equidistant instead of $\tau_i$, the assumptions of \cref{thm:fi} are satisfied with $\bar h=1/(N+1)$, so the coefficient of $I(\mu\mid\gamma_d)$ is improved while keeping the number of sampling steps $N$. 
Observe that this choice is inappropriate for discretizing the reverse SDE because $2(\tau_{i+1}-\tau_i)\geq 1-t_i/t_{i+1}=(i+2)^{-1}$ in this case.
\end{rmk}


\subsection{Adaptation to intrinsic dimensionality}

In practice, data are often distributed around a low-dimensional manifold. 
Recently, \citet{LiYa24} have shown that the DDPM sampler \eqref{eq:ddpm} must satisfy
\ben{\label{ddpm-ada}
\eta_i=1-\alpha_i,\qquad
\sigma_i^2=\frac{(1-\alpha_i)(\alpha_i-\ol{\alpha}_i)}{1-\ol{\alpha}_i}\qquad(i=0,1,\dots,N-1)
}
for $H(P^{\ddpm_N}\mid \slepian_{\alpha_N}\mu)$ having a bound independent of the ambient dimension $d$; see Theorem 2 ibidem. 
Interestingly, this choice matches the hyper-parameter setting optimized for the case that $\mu$ is the Dirac measure; see \cite[Section 3.2]{HJA20}.  
\citet{ADR24} and \citet{HWC26} give an interpretation for this specification from the SDE perspective. 
In particular, \citet{HWC26} have shown that this specification is interpreted as discretizing only the posterior mean component in the drift of \eqref{eq:score-diffusion}. 
For the F\"ollmer process, this corresponds to discretizing only $m(t,X_t)$ in the expression \eqref{vt-formula}. 
This leads to the following linear SDE:
\ben{\label{eq:sde-ada}
\begin{cases}
\wt X_0=0,\\
d\wt X_t=\frac{m(t_i,\wt X_{t_i})-\wt X_t}{1-t}dt+dW_t\quad\text{when }t\in[t_i,t_{i+1})\text{ for some }i=-1,0,\dots,N-1,
\end{cases}
}
where we set $t_{-1}:=0$ and $m(0,x)\equiv0$ by convention. 
Note that although \eqref{mt-formula} suggests setting $m(0,x)\equiv\E[X_1]$ when $\E[|X_1|]<\infty$, we always set $m(0,x)\equiv0$ because $\E[X_1]$ is unknown in practice. 
The above SDE can be explicitly solved as
\ben{\label{eq:ada-solved}
\wt X_t=\frac{1-t}{1-t_i}\bra{\wt X_{t_i}+\frac{t-t_i}{1-t}m(t_i,\wt X_{t_i})+\int_{t_i}^t\frac{1-t_i}{1-s}dW_s}\quad\text{for }t\in[t_i,t_{i+1}).
}
See e.g.~Section 5.6 in \cite{KaSh98}. 
Observe that
\[
\frac{1-t_{i+1}}{1-t_i}\int_{t_i}^t\frac{1-t_i}{1-s}dW_s\sim N\bra{0,(t_{i+1}-t_i)\frac{1-t_{i+1}}{1-t_i}I_d}\quad(i=-1,\dots,N-1)
\]
and they are independent. 
Hence the stochastic integral term of \eqref{eq:ada-solved} can be simulated exactly. 
Meanwhile, by \eqref{vt-formula} and the definition of $v_t$,
\[
m(t,x)=\frac{x}{t}+(1-t)\nabla\log p_t(x)\quad\text{for all }0<t<1,~x\in\mathbb R^d.
\]
Replacing the score function by its estimate, we obtain the following alternative discretization scheme:
\[
\begin{cases}
\ada_{t_0}=\sqrt{t_0(1-t_0)}Z_0,\\
\ada_{t_{i+1}}=\frac{t_{i+1}}{t_i}\ada_{t_i}+h_i\score(\ada_{t_i},t_i)+\sqrt{h_i\frac{1-t_{i+1}}{1-t_{i}}}Z_i\quad(i=0,1,\dots,N-1).
\end{cases}
\]
One can easily check that if we change the initialization in \eqref{eq:ddpm} to $\ddpm_0\sim N(0,\sqrt{1-t_0}I_d)$, $\ddpm_i=\ada_{t_i}/\sqrt{t_i}$ $(i=1,\dots,N)$ satisfy \eqref{eq:ddpm} with $\scored_i(x)=\sqrt{t_i}\score(\sqrt{t_i}x,t_i)$, $\alpha_i=t_i/t_{i+1}$ for $i=0,1,\dots N-1$, $\alpha_N=t_N$ and specification \eqref{ddpm-ada}. 
Below we will see that the different initialization serves as avoiding the appearance of the ambient dimension $d$ caused by initialization.

%
%
%

The following result gives a bound corresponding to \cref{thm:em} for $(\ada_{t_i})_{i=0}^N$. We prove it in \cref{proof:ada}. 
\begin{theorem}\label{thm:ada}
Assume \ref{ass:sampling}--\ref{ass:score} and $t_0\leq1/2$. Then
\[
H(P^{X_{t_N}}\mid P^{\ada_{t_N}})
\leq t_0\E[|X_1|^2]
+\kappa d\log\bra{\frac{1-t_0}{\delta}}
+3\kappa\E[|X_1|^2]
+\eps_\mathrm{score}^2.
\]
\end{theorem}

Next we state a bound adaptive to a low-dimensionality structure of data as in \cite{HWC26}. 
We need to introduce some notation. 
Given a subset $A\subset\mathbb R^d$ and $\eps>0$, we call a subset $\mcl N\subset A$ an \textit{$\eps$-net} of $A$ if $A\subset\bigcup_{x\in \mcl N}\mcl B(x;\eps)$, where $\mcl B(x;\eps):=\{y\in\mathbb R^d:|y-x|<\eps\}$. 
The minimal cardinality of $\eps$-nets of $A$ is called the \textit{covering number} of $A$ and denoted by $\mcl N(A,\eps)$. 
Also, let $\mcl X$ be the support of $\mu$. 

\begin{enumerate}[label={\normalfont[A\arabic*]},start=3]

\item\label{ass:intrinsic}
There exist constants $0<\eps_0\leq e^{-1}$ and $k\geq1$ such that
\ba{
\mcl N(\mcl X,\eps_0)\leq \eps_0^{-k},\qquad
\eps_0\leq\sqrt\delta\log(1/\eps_0)
}
and 
\ban{
\E\sbra{\sup_{x,y\in\mcl X:|x-y|\leq\eps_0}\abs{(x-y)\cdot Z}}&\leq\sqrt\delta k\log(1/\eps_0),\label{eq:gw}
}
where $Z\sim N(0,1)$. 

\item\label{ass:mom} $\E[|X_1|^4]<\infty$ and there exists a constant $R\geq3$ such that $\E[|X_1-\E[X_1]|^4]\leq R^4$. 

\end{enumerate}
\begin{rmk}
(a) \ref{ass:intrinsic} is essentially equivalent to Assumption 1 in \cite{HWC26}: A seemingly additional condition \eqref{eq:gw} is satisfied as long as $\eps_0$ is sufficiently small because we always have
\ben{\label{gw-naive}
\E\sbra{\sup_{x,y\in\mcl X:|x-y|\leq\eps_0}\abs{(x-y)\cdot Z}}
\leq\eps_0\E[|Z|]\leq\eps_0\sqrt d.
}
We keep the condition \eqref{eq:gw} in this form because the left hand side of \eqref{eq:gw} is known as the \textit{Gaussian width} of the set $\{x-y:x,y\in\mcl X,|x-y|\leq\eps_0\}$ whose bound is well-studied in the literature; see e.g.~\cite[Chapters 7--8]{Ve25}. See also \cref{coro:dim-free}. 

\noindent(b) \ref{ass:intrinsic} is closely related to the notion of \textit{Minkowski dimension} (see e.g.~\cite[Chapter 5]{Ma95} for a detailed account). 
For a subset $A\subset\mathbb R^d$, the \textit{upper Minkowski dimension} of $A$ is defined as
\[
\ol{\dim}_M(A):=\limsup_{\eps\downarrow0}\frac{\log\mcl N(A,\eps)}{\log(1/\eps)}. 
\]
Also, for $\eps>0$, the \textit{$\eps$-enlargement} of $A$ is defined as the set $\{y\in\mathbb R^d:|y-x|<\eps\text{ for some }x\in A\}$. 
One can easily check that \ref{ass:intrinsic} is satisfied if $\mcl X$ is contained in the $(\eps_0/2)$-enlargement of a set $A\subset\mathbb R^d$ with $\ol{\dim}_M(A)<k/(1+\log 2)$ for sufficiently small $\eps_0>0$ (see also \cite[Lemma 3]{HWC26}). 
We have $\ol{\dim}_M(A)\leq m$ for some $m\in\mathbb N$ when $A=g(B)$ for some bounded $B\subset\mathbb R^m$ and Lipschitz map $g:B\to\mathbb R^n$. This directly follows from definition once we note that $\ol{\dim}_M(B)\leq m$ for any bounded subset $B\subset\mathbb R^m$ by Eq.(4.17) in \cite{Ve25}. 
More precisely, if the diameter of $B$ is bounded by $D\geq1$ and the Lipschitz constant of $g$ is bounded by $L\geq1$, then $\mcl N(A,\eps)\leq\eps^{-m-\log(3DL)}$ for all $0<\eps<e^{-1}$. 
\citet{HWC26} give more examples satisfying \ref{ass:intrinsic}.

\noindent(c) \ref{ass:mom} is weaker than Assumption 2 in \cite{HWC26} because we always have $\E[|X_1-\E[X_1]|^4]\leq(2\sup_{x\in\mcl X}|x|)^4$. There exist situations where the left hand side of this inequality is substantially smaller than the right hand side. For instance, let $a\in\mathbb R^d$ and assume $\mcl X=\{0,a\}$, i.e., $\mu$ is a two-points distribution. Then, with $p:=P(X_1=a)$, we have $\sup_{x\in\mcl X}|x|=|a|$ while $\E[|X_1-\E[X_1]|^4]=p(1-p)\{p^3+(1-p)^3\}|a|^4$. 
Therefore, $\E[|X_1-\E[X_1]|^4]$ is arbitrarily smaller than $\sup_{x\in\mcl X}|x|$ as $p$ is closer to 0 or 1.
\end{rmk}

\begin{theorem}\label{thm:ada-m}
Assume \ref{ass:sampling}--\ref{ass:mom} and $t_0\leq1/2$.
There exists a universal constant $C>0$ such that
\ben{\label{eq:ada-m}
H(P^{X_{t_N}}\mid P^{\ada_{t_N}})
\leq t_0\E[|X_1|^2]
+4\kappa\E[|X_1|^2]
+C\kappa L\log\bra{\frac{1-t_0}{\delta}}
+\eps_\mathrm{score}^2,
}
where $L:=k\log(1/\eps_0)+\log (R/\delta)$. If we further assume
\ben{\label{eq:sampling-r}
h_i\leq6\kappa t_i\qquad(i=0,1,\dots,N-1),
}
then $4\kappa\E[|X_1|^2]$ in \eqref{eq:ada-m} can be replaced by $C\kappa L\log(1/t_0)$. 
\end{theorem}
This result is proven in \cref{proof:ada-m}. 
Observe that 
$
h_i/t_i\leq 2e^{\tau_{i+1}-\tau_i}(\tau_{i+1}-\tau_i)
$
for every $i$. Hence \eqref{eq:sampling-r} follows from \eqref{eq:tau-samp} if $\kappa\leq1$. 
To compare our bound to \cite{HWC26}'s one, we focus on the case that the order of $R,\delta,d$ and $\eps_0$ are controlled by $k$. 
Noting \cref{rmk:bddd} and \eqref{gw-naive}, we have the following corollary: 
\begin{corollary}\label{coro:hwc}
Assume \ref{ass:sampling}--\ref{ass:mom}, $t_0\leq1/2$ and \eqref{eq:sampling-r}. 
Moreover, suppose that there exists constants $c_1,c_2\geq1$ such that $R\leq k^{c_1k},\delta\geq k^{-c_2},d\leq k^{c_2}$ and $\eps_0= k^{1-c_2}$. 
Then we have
\[
H(P^{X_{t_N}}\mid P^{\ada_{t_N}})
\leq t_0\E[|X_1|^2]
+C_1\kappa\min\cbra{\E[|X_1|^2],(k\log k)\log(1/t_0)}
+C_1\kappa^2Nk\log k
+\eps_\mathrm{score}^2,
\]
where $C_1>0$ is a constant depending only on $c_1$ and $c_2$. 
\end{corollary} 
\begin{rmk}[Comparison to \citet{HWC26}]\label{rmk:hwc}
Under similar assumptions to \cref{coro:hwc}, Theorem 2 in \cite{HWC26} gives the same bound as in \cref{coro:hwc} but with an additional term $t_0d$ up to a constant. 
This improvement comes from the difference in their initialization steps. 
In fact, the standard initialization corresponds to using the Euler--Maruyama scheme to discretize the ``augmented'' process $(X_t)_{t\in[0,t_0]}$ in the F\"ollmer process, causing the $t_0\bra{d+\E[|X_1|^2]}$ error. 
We also remark that our assumption on $R$ is weaker than that of \cite{HWC26}. 
In fact, the latter assumes that $R$ is an upper bound of the diameter of $\mcl X$ and of polynomial growth in $k$. 
\end{rmk}

The following corollary states the bound in terms of the Minkowski dimension of $\mcl X$ and completely avoids any growth conditions on $d$. 
\begin{corollary}\label{coro:dim-free}
Assume \ref{ass:sampling}--\ref{ass:score}, \ref{ass:mom} and $t_0\leq1/2$. 
Also, assume that there exist constants $k,c\geq1$ such that $\mcl N(\mcl X,\eps)\leq c\eps^{-k}$ for all $0<\eps\leq e^{-1}$ and $R\leq c\delta^{-k}$. Then we have
\ben{\label{eq:dim-free}
H(P^{X_{t_N}}\mid P^{\ada_{t_N}})
\leq t_0\E[|X_1|^2]
+4\kappa\E[|X_1|^2]
+C_2\kappa k\log(1/\delta)\log\bra{\frac{1-t_0}{\delta}}
+\eps_\mathrm{score}^2,
}
where $C_2$ is a constant depending only on $c$. 
If we further assume \eqref{eq:sampling-r}, then $4\kappa\E[|X_1|^2]$ in \eqref{eq:dim-free} can be replaced by $C_2\kappa k\log(1/\delta)\log(1/t_0)$. 
\end{corollary} 

\begin{proof}
By Dudley's maximal inequality (e.g.~Theorem 2.2.4 in \cite{VaWe23}), there exists a universal constant $C>0$ such that
\[
\E\sbra{\sup_{x,y\in\mcl X:|x-y|\leq\eps_0}\abs{(x-y)\cdot Z}}\leq C\int_0^{\eps_0}\sqrt{\log\mcl N(\mcl X,\eps/2)}d\eps
\]
for any $\eps_0>0$. Therefore, for any $0<\eps_0\leq e^{-1}$, 
\ba{
\E\sbra{\sup_{x,y\in\mcl X:|x-y|\leq\eps_0}\abs{(x-y)\cdot Z}}
&\leq C\bra{\eps_0\sqrt{\log c}+\int_0^{\eps_0/2}\sqrt{k\log(1/\eps)} d\eps}
\leq C'\sqrt k\eps_0\sqrt{\log(1/\eps_0)},
}
where $C'\geq1$ is a constant depending only on $c$. 
Hence \ref{ass:intrinsic} is satisfied with $\eps_0=\min\{\sqrt{\delta}/C',e^{-1}\}$ and $k$ replaced by $k+\log c$. 
Consequently, the desired result follows from \cref{thm:ada-m}. 
\end{proof}

\appendix

\section{Additional proofs}

\subsection{Proof of Proposition \ref{prop:f-rc}}\label{proof:f-rc}

\begin{lemma}\label{lem:f-rc}
Let $g:\mathbb R^d\to\mathbb R$ be a bounded Borel function and $\nu$ a probability distribution on $\mathbb R^d$. 
Then, for any $s>0$,
\[
\int_{\mathbb R^d}g(y)e^{\frac{X_r\cdot y}{s-r}-\frac{r|y|^2}{2s(s-r)}}\nu(dy)\to^p\int_{\mathbb R^d}g(y)\nu(dy)\quad(r\downarrow0).
\]
\end{lemma}

\begin{proof}
Fix a constant $K>0$ arbitrarily. For any $r\in(0,s/2)$, we have on the event $\{|X_r|\leq K\sqrt r\}$
\ba{
\frac{X_r\cdot y}{s-r}-\frac{r|y|^2}{2s(s-r)}
\leq \frac{K\sqrt r|y|}{s-r}-\frac{r|y|^2}{2s(s-r)}
=\frac{-(\sqrt r|y|-sK)^2+s^2K^2}{2s(s-r)}
\leq K^2.
}
Therefore, by the bounded convergence theorem,
\[
1_{\{|X_r|\leq K\sqrt r\}}\int_{\mathbb R^d}g(y)\bra{e^{\frac{X_r\cdot y}{s-r}-\frac{r|y|^2}{2s(s-r)}}-1}\nu(dy)\to0\quad(r\downarrow0)\quad\text{a.s.}
\]
Meanwhile, observe that
\[
P(|X_r|> K\sqrt r)\leq P(r|X_1|>K\sqrt r/2)+P(|\beta_r|> K\sqrt r/2).
\]
By the Markov inequality, $P(|\beta_r|> K\sqrt r/2)\leq\frac{4}{K^2r}\E[|\beta_r|^2]\leq\frac{4}{K^2}$. Hence,
\[
\limsup_{r\downarrow0}P(|X_r|> K\sqrt r)\leq\frac{4}{K^2}.
\]
Consequently, for any $\eps>0$,
\ba{
\limsup_{r\downarrow0}P\bra{\abs{\int_{\mathbb R^d}g(y)e^{\frac{X_r\cdot y}{s-r}-\frac{r|y|^2}{2s(s-r)}}\nu(dy)-\int_{\mathbb R^d}g(y)\nu(dy)}>\eps}
\leq\frac{4}{K^2}.
}
Letting $K\to\infty$, we obtain the desired result. 
\end{proof}

\begin{proof}[\bf\upshape Proof of \cref{prop:f-rc}]
We follow the same proof strategy as \cite[Proposition III.2.10]{ReYo99}. 
It suffices to prove 
\ben{\label{frc:aim}
\E\sbra{\xi\mid\mcl F_t^X}=\E\sbra{\xi\mid\mcl  F_{t+}^X}
}
for any $t\in[0,1)$ and any $\sigma(X_s:s\in[0,1])$-measurable bounded random variable $\xi$. 
By the monotone class theorem, we may assume that $\xi$ is of the form $\xi=\prod_{i=1}^ng_i(X_{s_i})$ for some $0=:s_0<s_1<\cdots<s_n\leq1$ and bounded Borel functions $g_i:\mathbb R^d\to\mathbb R$, $i=1,\dots,n$. 
Let $k$ be the unique index such that $s_{k-1}\leq t<s_k$. 
We separately consider the following two cases:

\textbf{Case 1:} $t>0$. 
By \cref{prop:markov}, we have for any $r\in[t,s_k)$
\ben{\label{eq:markov}
\E\sbra{\xi\mid\mcl F_{r}^X}
=\bra{\prod_{i=1}^{k-1}g_i(X_{s_i})}G_r(X_{r}),
}
where
\[
G_r(x)=\int g_k(x_k)P^X_{r,s_k}(x,dx_k)\int g_{k+1}(x_{k+1})P^X_{s_k,s_{k+1}}(x_k,dx_{k+1})\cdots\int g_n(x_{n})P^X_{s_{n-1},s_{n}}(x_{n-1},dx_{n}).
\]
From the expression \eqref{eq:trans}, we can easily see that, as $r\downarrow t$, $G_r$ converges uniformly on every compact set of $\mathbb R^d$ to $G_t$. Moreover, $X_{r}\to X_t$ as $r\downarrow t$ a.s.~by continuity. 
Hence $G_r(X_{r})\to G_t(X_t)$ as $r\downarrow t$ a.s. 
Therefore, $\E\sbra{\xi\mid\mcl F_{r}^X}\to\E\sbra{\xi\mid\mcl F_{t}^X}$ as $r\downarrow t$ a.s. 
Meanwhile, $\E\sbra{\xi\mid\mcl F_{r}^X}\to\E\sbra{\xi\mid\mcl F_{t+}^X}$ as $r\downarrow t$ a.s.~by \cite[Corollary III.2.4]{ReYo99}. Hence we obtain \eqref{frc:aim}. 

\textbf{Case 2:} $t=0$. 
In this case, we have $k=1$ and \eqref{eq:markov} still holds for $r\in(0,s_1)$. 
Observe that
\ba{
\frac{\phi_{\frac{r}{s_1}(s_1-r)I_d}(x-\frac{r}{s_1}y)}{p_{r}(x)}
&=\bra{\frac{1-r}{1-r/s_1}}^{d/2}e^{-\frac{1-s_1}{2(s_1-r)(1-r)}|x|^2}\frac{e^{\frac{x\cdot y}{s_1-r}-\frac{r|y|^2}{2s_1(s_1-r)}}}{\int e^{\frac{x\cdot z}{1-r}-\frac{r|z|^2}{2(1-r)}}\mu(dz)}.
}
By this expression and \cref{lem:f-rc}, 
\ba{
G_r(X_r)\to^p\int g_1(x_1)P^{X_{s_1}}(dx_1)\int g_{2}(x_{2})P^X_{s_1,s_{2}}(x_1,dx_{2})\cdots\int g_n(x_{n})P^X_{s_{n-1},s_{n}}(x_{n-1},dx_{n})
}
as $r\downarrow 0$. The last expression is equal to $\E[\xi]$. 
Hence, \eqref{frc:aim} follows from \eqref{eq:markov} and \cite[Corollary III.2.4]{ReYo99}.
\end{proof}

\subsection{Proof of Theorem \ref{thm:weak}}\label{proof:weak}

We need two auxiliary results. 
The first one is the so-called \textit{Tweedie's formula} \cite{Ef11,Ef24} that expresses scores as conditional expectations. 
\begin{lemma}[Tweedie's formula]
Let $Y$ and $Z$ be independent random vectors in $\mathbb R^d$ such that $Z\sim N(0,\Sigma)$ for some positive definite symmetric matrix $\Sigma$. 
Then $Y+Z$ has a positive, smooth density $f_{Y+Z}$ and
\[
\nabla\log f_{Y+Z}(Y+Z)=-\E[\Sigma^{-1}Z\mid Y+Z].
\]
\end{lemma}

\begin{proof}
It is well-known that $f_{Y+Z}(x)=\E[\phi_{\Sigma}(x-Y)]$ is the density of $Y+Z$. 
So the first claim follows. 
We also obtain
\[
\E[|\nabla\log f_{Y+Z}(Y+Z)|]
=\int_{\mathbb R^d}|\nabla f_{Y+Z}(x)|dx
\leq\E\sbra{\int_{\mathbb R^d}|\nabla \phi_{\Sigma}(x-Y)|dx}
=\int_{\mathbb R^d}|\nabla \phi_{\Sigma}(x)|dx<\infty.
\]
Hence the second claim follows once we show that
\[
\E[\nabla\log f_{Y+Z}(Y+Z)g(Y+Z)]=-\E[\Sigma^{-1}Zg(Y+Z)]
\] 
for any compactly supported smooth function $g:\mathbb R^d\to\mathbb R$. 
Integration by parts gives
\ba{
\E[\nabla\log f_{Y+Z}(Y+Z)g(Y+Z)]
=\int_{\mathbb R^d}\nabla f_{Y+Z}(x)g(x)dx
=-\int_{\mathbb R^d} f_{Y+Z}(x)\nabla g(x)dx.
}
Meanwhile, by the multivariate Stein identity
\[
\E[\Sigma^{-1}Zg(Y+Z)]=\E[\nabla g(Y+Z)]=\int_{\mathbb R^d} \nabla g(x)f_{Y+Z}(x)dx.
\]
Hence we obtain the desired result. 
\end{proof}

The second one is a consequence of the Markov property of $X$:
\begin{lemma}\label{lem:beta}
$\E[\beta_t\mid\mcl F_t^X]=\E[\beta_t\mid X_t]$ for all $t\in[0,1]$. 
\end{lemma}

\begin{proof}
Define $\eta_n:=\beta_t1_{\{|\beta_t|\leq n\}}$ for each $n\in\mathbb N$. Since $X$ is a Markov process and $\eta_n$ is bounded and measurable with respect to $\sigma(X_s,s\in[t,1])$, we have $\E[\eta_n\mid\mcl F^X_t]=\E[\eta_n\mid X_t]$ by property (iia) in \cite[Section 1.1]{ChWa05}. 
Since $\eta_n\to\beta_t$ as $n\to\infty$, $|\eta_n|\leq|\beta_t|$ and $\E[|\beta_t|]<\infty$, we obtain the desired result by the dominated convergence theorem. 
\end{proof}

\begin{proof}[\bf\upshape Proof of \cref{thm:weak}]
Observe that
\[
\E\sbra{\int_0^1\frac{|\beta_t|}{1-t}dt}
\leq\sqrt d\int_0^1\frac{\sqrt t}{\sqrt{1-t}}dt<\infty.
\]
Hence we can define a process $B=(B_t)_{t\in[0,1]}$ as
\ben{\label{bb-bm}
B_t=\beta_t+\int_0^t\frac{\beta_s}{1-s}ds\quad(0\leq t\leq1).
}
Recall that $\beta$ is a unique strong solution to the following SDE:
\[
d\beta_t=\frac{-\beta_t}{1-t}dt+dW_t\quad(0\leq t<1),\qquad \beta_0=0.
\]
This fact and the continuity of $B$ at $t=1$ imply that $B$ is a standard Wiener process in $\mathbb R^d$ such that $\mathbf F^B=\mathbf F^\beta$.  

Now, by Tweedie's formula and \cref{lem:beta},
\ben{\label{eq:tweedie}
\nabla \log p_s(X_s)=-\frac{\E[\beta_s\mid X_s]}{s(1-s)}=-\frac{\E[\beta_s\mid \mcl F^X_s]}{s(1-s)}.
}
Hence
\ba{
W^X_t=tX_1+\beta_t-\int_0^t\bra{X_1+\frac{\beta_s}{s}-\frac{\E[\beta_s\mid \mcl F^X_s]}{\sqrt{s(1-s)}}}ds
=B_t-\int_0^t\bra{\frac{\beta_s}{s(1-s)}-\frac{\E[\beta_s\mid \mcl F^X_s]}{s(1-s)}}ds.
}
This expression shows that $\E[|W^X_t|]<\infty$. 
Moreover, for $0\leq s<t\leq1$, $\E[B_t-B_s\mid\mcl F^X_s]=\E[\E[B_t-B_s\mid\mcl F^\beta_s\vee\sigma(X_1)]\mid\mcl F^X_s]=0$ because $\beta$ is independent of $X_1$ and $B$ is a standard $\mathbf F^\beta$-Wiener process in $\mathbb R^d$. 
Therefore, 
\ba{
\E[W^X_t-W^X_s\mid\mcl F_s^X]
=\int_s^t\bra{\frac{\E[\beta_u\mid\mcl F^X_s]}{u(1-u)}-\frac{\E[\E[\beta_u\mid\mcl F^X_u]\mid \mcl F^X_s]}{u(1-u)}}du
=0.
}
Since $W^X$ is continuous and $\mathbf F^X$-adapted by construction, this means that $W^X$ is a continuous $\mathbf F^X$-martingale. 
Also, $[W^X,W^X]_t=[B,B]_t=tI_d$ for all $t\in[0,1]$. 
Hence $W^X$ is a standard $\mathbf F^X$-Wiener process in $\mathbb R^d$ by L\'evy's characterization. 
\end{proof}

\subsection{Proof of Proposition \ref{ent-formula}}\label{sec:ent-formula}

The proof is an application of Girsanov's theorem. 
For later use, we formulate it in a general form. 
Let $w=(w_t)_{t\in[0,1]}$ be the canonical process on $\mathbb W^d$. 
Set $\mcl B_t=\sigma(w_s:s\leq t)$ for $t\in[0,1]$. 
For two probability distributions $P$ and $Q$ on $\mathbb W^d$ and $t\in[0,1]$, we write
$
H_t(P\mid Q):=H(P|_{\mcl B_t}\mid Q|_{\mcl B_t})
$
for short.

\begin{lemma}\label{lem:girsanov}
For each $i=1,2$, let $b_i:(0,1)\times\mathbb W^d\to\mathbb R^d$ be a progressively measurable functional. 
Suppose that $X^{(i)}=(X^{(i)}_t)_{t\in[0,1]}$ be a weak solution to the following SDE:
\[
dX^{(i)}_t=b_i(t,X^{(i)})dt+dW_t,\qquad X^{(i)}_0=0.
\]
Suppose also that for all $t\in[0,1)$ 
\ben{\label{ass:girsanov}
\begin{cases}
P^{X^{(1)}}\bra{\int_0^t|b_1(s,w)|^2ds<\infty}
=P^{X^{(2)}}\bra{\int_0^t|b_2(s,w)|^2ds<\infty}=1,\\
P^{W}\bra{\int_0^t|b_2(s,w)|^2ds<\infty}=1.
\end{cases}
}
Then 
\ben{\label{eq:girsanov}
H_t(P^{X^{(1)}}\mid P^{X^{(2)}})
=\frac{1}{2}\int_0^{t}\E\sbra{|b_1(s,X^{(1)})-b_2(s,X^{(1)})|^2}ds
}
for all $t\in[0,1)$. 
Moreover, if \eqref{ass:girsanov} holds with $t=1$, \eqref{eq:girsanov} also holds with $t=1$. 
\end{lemma}

\begin{proof}
The second claim can be shown in a similar manner to the first one, so we focus on the proof of the first claim. 
By Theorem 7.6 in \cite{LiSh01}, $P^{X^{(1)}}|_{\mcl B_t}\ll P^W|_{\mcl B_t}$ for any $t\in[0,1)$ and the process $(\frac{dP^{X^{(1)}}|_{\mcl B_t}}{dP^W|_{\mcl B_t}})_{t\in[0,1)}$ has a continuous modification $(D^{(1)}_t)_{t\in[0,1)}$ with respect to $P^W$ such that
\[
D^{(1)}_t(w)=\exp\bra{\int_0^{t}b_1(s,w)dw_s-\frac{1}{2}\int_0^{t}|b_1(s,w)|^2ds}\quad P^{X^{(1)}}\text{-a.s.}
\]
Similarly, by Theorem 7.7 in \cite{LiSh01}, $P^{X^{(2)}}|_{\mcl B_t}\sim P^W|_{\mcl B_t}$ for any $t\in[0,1)$ and the process $(\frac{dP^{X^{(2)}}|_{\mcl B_t}}{dP^W|_{\mcl B_t}})_{t\in[0,1)}$ has a continuous modification $(D^{(2)}_t)_{t\in[0,1)}$ with respect to $P^W$ such that
\ben{\label{eq:D2}
D^{(2)}_t(w)=\exp\bra{\int_0^{t}b_2(s,w)dw_s-\frac{1}{2}\int_0^{t}|b_2(s,w)|^2ds}\quad P^W\text{-a.s.}
}
Since $P^{X^{(1)}}|_{\mcl B_t}\ll P^W|_{\mcl B_t}$, Eq.\eqref{eq:D2} also holds $P^{X^{(1)}}$-a.s. 
Therefore, with $\Lambda_t:=D^{(1)}_t/D^{(2)}_t$, 
\[
\Lambda_t(w)=\exp\bra{\int_0^{t}\bra{b_1(s,w)-b_2(s,w)}dw_s-\frac{1}{2}\int_0^{t}\bra{|b_1(s,w)|^2-|b_2(s,w)|^2}ds}\quad P^{X^{(1)}}\text{-a.s.}
\]
Hence
\ben{\label{eq:loglik}
\log\Lambda_t(X^{(1)})=\frac{1}{2}\int_0^{t}|b_1(s,X^{(1)})-b_2(s,X^{(1)})|^2ds
+\int_0^{t}\bra{b_1(s,X^{(1)})-b_2(s,X^{(1)})}dW_s\quad\text{a.s.}
}
Meanwhile, by the definition of $D^{(1)}$ and $D^{(2)}$,
\ben{\label{eq:lr}
\Lambda_t=\frac{dP^{X^{(1)}}|_{\mcl B_t}}{dP^{X^{(2)}}|_{\mcl B_t}}\quad P^W\text{-a.s.~and~}P^{X^{(2)}}\text{-a.s.}
}
Since $P^{X^{(1)}}|_{\mcl B_t}\ll P^W|_{\mcl B_t}$, this also holds $P^{X^{(1)}}$-a.s. Hence
\ben{\label{eq:ent-g}
H_t(P^{X^{(1)}}\mid P^{X^{(2)}})=\E[\log\Lambda_t(X^{(1)})].
}
Now, if the right hand side of \eqref{eq:girsanov} is finite, then
\[
\E\sbra{\int_0^{t}\bra{b_1(s,X^{(1)})-b_2(s,X^{(1)})}dW_s}=0.
\]
Hence \eqref{eq:loglik} and \eqref{eq:ent-g} give \eqref{eq:girsanov}. 
Otherwise, for every $n\in\mathbb N$, define a function $\tau_n:\mathbb W^d\to[0,1]$ as $\tau_n(w)=\inf\{t\in[0,1]:\int_0^t(|b_1(s,w)|^2+|b_2(s,w)|^2)ds\geq n\}\wedge1$ for $w\in\mathbb W^d$. 
We can easily check that $\tau_n$ is a $(\mcl B_t)_{t\in[0,1]}$-stopping time. 
Also, from \eqref{eq:lr} we can verify that $(\Lambda_t)_{t\in[0,1)}$ is a $(\mcl B_t)_{t\in[0,1)}$-martingale under $P^{X^{(2)}}$. 
Hence we have $\Lambda_{t\wedge\tau_n}=\E_{P^{X^{(2)}}}[\Lambda_t\mid\mcl B_{t\wedge\tau_n}]$ $P^{X^{(2)}}$-a.s.~for any $t\in[0,1)$ by the optional sampling theorem. 
This and \eqref{eq:lr} show that
\ben{\label{eq:lr-stop}
\Lambda_{t\wedge\tau_n}=\frac{dP^{X^{(1)}}|_{\mcl B_{t\wedge\tau_n}}}{dP^{X^{(2)}}|_{\mcl B_{t\wedge\tau_n}}}\quad P^{X^{(2)}}\text{-a.s.}
}
Also, since $(0,\infty)\ni x\mapsto x\log x\in\mathbb R$ is convex, Jensen's inequality gives
\[
\int\Lambda_{t\wedge\tau_n}\log\Lambda_{t\wedge\tau_n}dP^{X^{(2)}}
\leq\int\Lambda_t\log\Lambda_{t}dP^{X^{(2)}}.
\]
Combining this with \eqref{eq:lr} and \eqref{eq:lr-stop} gives
\ben{\label{eq:ent-stop}
\E[\log\Lambda_{t\wedge\tau_n}(X^{(1)})]
\leq\E[\log\Lambda_{t}(X^{(1)})].
}
Meanwhile, by \eqref{eq:loglik} and the definition of $\tau_n$,
\[
\E[\log\Lambda_{t\wedge\tau_n}(X^{(1)})]=\frac{1}{2}\E\sbra{\int_0^{t\wedge\tau_n(X^{(1)})}|b_1(s,X^{(1)})-b_2(s,X^{(1)})|^2ds}.
\]
Noting $P^{X^{(1)}}|_{\mcl B_t}\ll P^W|_{\mcl B_t}$, we have $\tau_n(X^{(1)})\uparrow1$ as $n\to\infty$ a.s.~by \eqref{ass:girsanov}. 
Hence the monotone convergence theorem gives 
\[
\lim_{n\to\infty}\E[\log\Lambda_{t\wedge\tau_n}(X^{(1)})]=\frac{1}{2}\E\sbra{\int_0^{t}|b_1(s,X^{(1)})-b_2(s,X^{(1)})|^2ds}=\infty.
\]
Consequently, we have $H_t(P^{X^{(1)}}\mid P^{X^{(2)}})=\infty$ by \eqref{eq:ent-g} and \eqref{eq:ent-stop}, which means \eqref{eq:girsanov}. 
\end{proof}

%

\begin{proof}[\bf Proof of \cref{ent-formula}]
Note that $\mu\ll\gamma_d$ if and only if the law of $X$ is absolutely continuous with respect to $P^W$. 
By Theorem 7.5 in \cite{LiSh01}, this is also equivalent to $\int_0^1|v_t|^2dt<\infty$ a.s. 
Therefore, if $\mu\not\ll\gamma_d$, both sides of \eqref{eq:ent-formula} are infinite and thus \eqref{eq:ent-formula} holds. 
If $\mu\ll\gamma_d$, \cref{lem:girsanov} gives 
\[
H(P^X\mid P^W)=\E\sbra{\frac{1}{2}\int_0^1|v_s|^2ds}.
\] 
Since $\frac{dP^X}{dP^W}(X)=\frac{d\mu}{d\gamma_d}(X_1)$ by the definition of the F\"ollmer process, this shows \eqref{eq:ent-formula}. 
\end{proof}

\subsection{Proofs of Proposition \ref{prop:vt-mar}}\label{proof:vt-mar}

First, by \eqref{vt-formula}, \eqref{mt-formula} and the Markov property of $X$, we can rewrite $v_t$ as
\ben{\label{eq:vt-mar}
v_t=\frac{\E[X_1\mid\mcl F^X_t]-X_t}{1-t}=\E[X_1\mid\mcl F^X_t]-\frac{\E[\beta_t\mid\mcl F^X_t]}{1-t}.
}
Next, let $B=(B_t)_{t\in[0,1]}$ be the process defined by \eqref{bb-bm}. 
By the proof of \cref{thm:weak}, $B$ is a standard $\mathbf F^\beta$-Wiener process in $\mathbb R^d$. 
Moreover, solving \eqref{bb-bm} with respect to $\beta$, we obtain
\[
\beta_t=(1-t)\int_0^t\frac{1}{1-s}dB_s
\]
for all $t\in[0,1)$ (cf.~\cite[Section 5.6.B]{KaSh98}). This implies that $(\beta_t/(1-t))_{t\in[0,1)}$ is an $\mathbf F^\beta$-martingale in $\mathbb R^d$. 
Since $\beta$ is independent of $X_1$ and $\mcl F^X_t\subset\mcl F^{\beta}_t\vee\sigma(X_1)$ for all $t\in[0,1]$, 
combining this with the expression \eqref{eq:vt-mar} gives the desired result. 
\qed

\subsection{Proofs of Proposition \ref{prop:vt-mom}}\label{proof:vt-mom}

By Tweedie's formula, we can rewrite $v_t$ as
\[
v_t=X_1+\frac{\beta_t}{t}-\frac{\E[\beta_t\mid X_t]}{t(1-t)}.
\]
Since $\E[|\beta_t|^r]<\infty$, the asserted claim follows from this expression. 
\qed

\subsection{Proof of \eqref{eq:score-match}}\label{proof:score-match}

By \eqref{eq:tweedie}, we have
\[
\E\sbra{\nabla\log p_{t}(X_{t})-\frac{-\beta_t}{t(1-t)}\mid X_t}=0.
\]
Hence we obtain
\be{
\E\sbra{\bra{\frac{-\beta_t}{t(1-t)}-\nabla\log p_{t}(X_{t})}\cdot\bra{\nabla\log p_{t}(X_{t})-\score(X_{t},t)}}=0.
}
Since $(\beta_t/\{t(1-t)\},X_t)$ has the same law as $(Z/\sqrt{t(1-t)},\xi_t)$ by \cref{prop:basic}, we obtain the desired result from the above identity. \qed

\subsection{Proof of Theorem \ref{thm:em}}\label{proof:em}

The following lemma is a special case of the so-called data processing inequality and well-known in information theory (see e.g.~Corollary 2.18 in \cite{PoWu25}). 
We give a direct proof for the sake of readers' convenience. 
\begin{lemma}\label{lem:dpi}
Let $(\mcl X,\mcl A)$ and $(\mcl T,\mcl B)$ be two measurable spaces and $T:\mcl X\to\mcl T$ an $\mcl A/\mcl B$-measurable function. 
For any two probability measures $P$ and $Q$ on $(\mcl X,\mcl A)$,
\[
H(P^T\mid Q^T)\leq H(P\mid Q),
\]
where $P^T$ and $Q^T$ are the image measures of $P$ and $Q$ by $T$, respectively. 
\end{lemma}

\begin{proof}
It suffices to consider the case $H(P\mid Q)<\infty$. 
In this case, we have $P\ll Q$, so $P^T\ll Q^T$. 
Moreover, 
\ben{\label{image-density}
\frac{dP^T}{dQ^T}\circ T=\E_Q\sbra{\frac{dP}{dQ}\mid T}\quad Q\text{-a.s.}
}
In fact, the left hand side of \eqref{image-density} is evidently $\sigma(T)$-measurable and for any $B\in\mcl B$
\ba{
\int_{T^{-1}(B)}\bra{\frac{dP^T}{dQ^T}\circ T}dQ
=\int1_B(y)\frac{dP^T}{dQ^T}(y)Q^T(dy)
=P(T^{-1}(B))
=\int_{T^{-1}(B)}\frac{dP}{dQ}dQ.
}
Since $\sigma(T)=\{T^{-1}(B):B\in\mcl B\}$, this implies \eqref{image-density}. 
Hence we obtain
\ba{
H(P^T\mid Q^T)
&=\int\frac{dP^T}{dQ^T}\log\bra{\frac{dP^T}{dQ^T}}dQ^T
=\int\E_Q\sbra{\frac{dP}{dQ}\mid T}\log\bra{\E_Q\sbra{\frac{dP}{dQ}\mid T}}dQ\\
&\leq\int\E_Q\sbra{\frac{dP}{dQ}\log\bra{\frac{dP}{dQ}}\mid T}dQ,
}
where the last inequality follows by Jensen's inequality and the fact that $x\mapsto x\log x$ is convex. Since the last term is equal to $H(P\mid Q)$, we obtain the desired result. 
\end{proof}

\begin{proof}[\bf Proof of \cref{thm:em}]
It suffices to consider the case $\E[|X_1|^2]<\infty$.  
First, note that \eqref{eq:v2} yields $\E[\int_0^t|v_s|^2ds]<\infty$ for all $t\in[0,1)$. 
Next, define a function $b:[0,1)\times\wiener^d\to\mathbb R^d$ as
\[
b(t,x)=\begin{cases}
0 & \text{if }0\leq t<t_0\text{ or }t_N\leq t<1,\\
x(t_{i})/t_{i}+\score(x(t_{i}),t_{i}) & \text{if }t_{i}\leq t<t_{i+1}\text{ for some }i=0,1,\dots,N-1.
\end{cases}
\]
Let $\hat X=(\hat X_t)_{t\in[0,1]}$ be a solution to the following SDE:
\[
d\hat X_t=b(t,\hat X)dt+dW_t,\quad\hat X_0=0.
\]
Then $\hat X_{t_N}$ has the same law as $\euler_{t_N}$. 
Hence, by \cref{lem:dpi}
\be{
H(P^{X_{t_N}}\mid P^{\euler_{t_N}})
\leq H_{t_N}(P^X\mid P^{\hat X}).
}
Also, by \cref{lem:girsanov},
\ba{
H_{t_N}(P^X\mid P^{\hat X})
&=\frac{1}{2}\int_0^{t_N}\E\sbra{|v_t-b(t,X)|^2}dt\\
&=\frac{1}{2}\int_0^{t_0}\E[|v_t|^2]dt
+\frac{1}{2}\sum_{i=0}^{N-1}\int_{t_{i}}^{t_{i+1}}\E\sbra{\abs{v_t-\frac{X_{t_{i}}}{t_{i}}-\score(t_{i},X_{t_{i}})}^2}dt.
}
Therefore, by \eqref{score-bound},
\ben{\label{eq:kl-basic-1}
H(P^{X_{t_N}}\mid P^{\euler_{t_N}})
\leq\frac{1}{2}\int_0^{t_0}\E[|v_t|^2]dt+\sum_{i=0}^{N-1}\int_{t_{i}}^{t_{i+1}}\E[|v_t-v_{t_i}|^2]dt+\eps_\text{score}^2.
}
Since $(v_t)_{t\in(0,1)}$ is an $\mathbf F^X$-martingale in $\mathbb R^d$ by \cref{prop:vt-mar}, we obtain
\ban{
H(P^{X_{t_N}}\mid P^{\euler_{t_N}})
&\leq\frac{1}{2}\int_0^{t_0}\E[|v_t|^2]dt
+\sum_{i=0}^{N-1}\int_{t_{i}}^{t_{i+1}}\bra{\E[|v_t|^2]-\E[|v_{t_i}|^2]}dt+\eps_\text{score}^2\label{eq:kl-basic0}\\
&\leq\frac{t_0}{2}\E[|v_{t_0}|^2]
+\sum_{i=0}^{N-1}h_i\bra{\E[|v_{t_{i+1}}|^2]-\E[|v_{t_{i}}|^2]}+\eps_\text{score}^2.\label{eq:kl-basic}
}
Using \eqref{eq:sampling}, we deduce
\ba{
\sum_{i=0}^{N-1}h_i\bra{\E[|v_{t_{i+1}}|^2]-\E[|v_{t_{i}}|^2]}
&\leq\kappa\sum_{i=0}^{N-1}(1-t_{i+1})\bra{\E[|v_{t_{i+1}}|^2]-\E[|v_{t_{i}}|^2]}\\
&=\kappa\sum_{i=1}^{N}(1-t_{i})\E[|v_{t_{i}}|^2]
-\kappa\sum_{i=0}^{N-1}(1-t_{i+1})\E[|v_{t_{i}}|^2]\\
&=\kappa\sum_{i=0}^{N-1}h_i\E[|v_{t_{i}}|^2]
+\kappa(1-t_{N})\E[|v_{t_{N}}|^2]-\kappa(1-t_0)\E[|v_{t_{0}}|^2]\\
&\leq\kappa\int_{t_0}^{t_N}\E[|v_{t}|^2]dt
+\kappa(1-t_{N})\E[|v_{t_{N}}|^2].
}
Now, by \eqref{eq:v2} and $t_0\leq1/2$,
\ba{
\E[|v_{t_{0}}|^2]\leq d\frac{t_0}{1-t_0}+\E[|X_1|^2]
\leq d+\E[|X_1|^2]
}
and
\ba{
(1-t_{N})\E[|v_{t_{N}}|^2]\leq d+\delta(\E[|X_1|^2]-d)
\leq d+\E[|X_1|^2]
}
and
\ba{
\int_{t_0}^{t_N}\E[|v_{t}|^2]dt
\leq d\int_{t_0}^{t_N}\frac{1}{1-t}dt+\E[|X_1|^2]-d
= d\log\bra{\frac{1-t_0}{\delta}}+\E[|X_1|^2]-d.
}
Consequently,
\ba{
\frac{t_0}{2}\E[|v_{t_0}|^2]
+\sum_{i=0}^{N-1}h_i\bra{\E[|v_{t_{i+1}}|^2]-\E[|v_{t_{i}}|^2]}
\leq\frac{t_0}{2}\bra{d+\E[|X_1|^2]}+\kappa d\log\bra{\frac{1-t_0}{\delta}}+2\kappa\E[|X_1|^2].
}
Combining this with \eqref{eq:kl-basic} gives the desired result. 
\end{proof}

\subsection{Proof of Theorem \ref{thm:fi}}\label{proof:fi}

The following lemma is equivalent to the exponential decay property of the Fisher information along the Ornstein--Uhlenbeck semigroup (see Eq.(5.7.4) of \cite{BGL14}). 
We give a direct proof for readers' convenience. 
\begin{lemma}\label{lem:fi}
If $\mu$ has an almost everywhere differentiable density $f$ with respect to $\gamma_d$, then $I(\slepian_t\mu\mid\gamma_d)/t\leq I(\mu\mid\gamma_d)$ for all $t\in(0,1)$.
\end{lemma}

\begin{proof}
We may assume $I(\mu\mid\gamma_d)<\infty$ without loss of generality. 
Set $g:=\sqrt{f}$. Then we have
\ben{\label{fi-sqrt}
\|g\|_{L^2(\gamma_d)}^2=\int_{\mathbb R^d}fd\gamma_d=1<\infty
\quad\text{and}\quad
\||\nabla g|\|^2_{L^2(\gamma_d)}
=\frac{1}{4}\int_{\mathbb R^d}\frac{|\nabla f|^2}{f}d\gamma_d
=\frac{1}{4}I(\mu\mid\gamma_d)<\infty.
}
Therefore, by Proposition 0.2 in the supplement to \cite{FGRS22}, there exist compactly supported $C^\infty$ functions $g^{(n)}:\mathbb R^d\to\mathbb R$ $(n=1,2,\dots)$ such that $\|g^{(n)}-g\|_{L^2(\gamma_d)}+\||\nabla g^{(n)}-\nabla g|\|_{L^2(\gamma_d)}\to0$ as $n\to\infty$. 

For each $n\in\mathbb N$, define a function $f^{(n)}_t:\mathbb R^d\to\mathbb R$ as
\[
f^{(n)}_t(y)
=\phi_{I_d}(y)^{-1}\int_{\mathbb R^d}\phi_{(1-t)I_d}(y-\sqrt tx)g^{(n)}(x)^2\gamma_d(dx)\quad(y\in\mathbb R^d).
\]
Since $\|g^{(n)}-g\|_{L^2(\gamma_d)}\to0$ as $n\to\infty$, we have $f^{(n)}_t(y)\to f_t(y)$ and $\nabla f^{(n)}_t(y)\to \nabla f_t(y)$ as $n\to\infty$. 
Hence Fatou's lemma yields
\ben{\label{fi-fatou}
I(\slepian_t\mu\mid\gamma_d)\leq\liminf_{n\to\infty}\int_{\mathbb R^d}\frac{|\nabla f^{(n)}_t(y)|^2}{f^{(n)}_t(y)}\gamma_d(dy).
}
Meanwhile, $f^{(n)}_t(y)$ can be rewritten as
\ba{
f^{(n)}_t(y)
&=\int_{\mathbb R^d}g^{(n)}(x)^2\phi_{(1-t)I_d}(x-\sqrt ty)dx
=\int_{\mathbb R^d}g^{(n)}(\sqrt ty+\sqrt{1-t}z)^2\gamma_{d}(dz).
}
Since $g^{(n)}$ is a compactly supported smooth function, we can differentiate under the integral sign on the right hand side and obtain
\[
\nabla f^{(n)}_t(y)=\sqrt t\int_{\mathbb R^d}(2g^{(n)}\nabla g^{(n)})(\sqrt ty+\sqrt{1-t}z)\gamma_d(dz).
\]
Therefore, by the Schwarz inequality,
\ba{
|\nabla f^{(n)}_t(y)|^2
&\leq 4tf^{(n)}_t(y)\int_{\mathbb R^d}|\nabla g^{(n)}(\sqrt t y+\sqrt{1-t}z)|^2\gamma_d(dz).
}
Hence
\ba{
\liminf_{n\to\infty}\int_{\mathbb R^d}\frac{|\nabla f^{(n)}_t(y)|^2}{f^{(n)}_t(y)}\gamma_d(dy)
&\leq 4t\liminf_{n\to\infty}\int_{\mathbb R^d}|\nabla g^{(n)}(z)|^2\gamma_d(dz)\\
&=4t\int_{\mathbb R^d}|\nabla g(z)|^2\gamma_d(dz)=tI(\mu\mid\gamma_d),
}
where the first equality follows by $\||\nabla g^{(n)}-\nabla g|\|_{L^2(\gamma_d)}\to0$ and the second by \eqref{fi-sqrt}. 
Combining this with \eqref{fi-fatou} gives the desired result. 
\end{proof}

The above lemma and de Bruijn’s identity \eqref{debruijn} immediately yield the following corollary:
\begin{corollary}[Gaussian logarithmic Sobolev inequality]
\[
H(\mu\mid\gamma_d)\leq\frac{1}{2}I(\mu\mid\gamma_d).
\]
\end{corollary}

\begin{proof}[\bf Proof of \cref{thm:fi}]
We may assume $I(\mu\mid\gamma_d)<\infty$ without loss of generality. 
Observe that $\E[\int_0^1|v_s|^2ds]<\infty$ by \cref{ent-formula} and the Gaussian logarithmic Sobolev inequality. 
This particularly implies $\E[|X_1|^2]<\infty$ by \cref{prop:vt-mom}. 
These facts allow us to obtain 
\ben{\label{eq:non-stop}
H(P^{X_1}\mid P^{\euler_1})\leq\frac{1}{2}\int_0^{t_0}\E[|v_t|^2]dt
+\sum_{i=0}^{N-1}\int_{t_{i}}^{t_{i+1}}\bra{\E[|v_t|^2]-\E[|v_{t_i}|^2]}dt+\eps_\text{score}^2
}
by a similar argument to the derivation of \eqref{eq:kl-basic0}. 
Since $(v_t)_{t\in[0,1)}$ is an $\mathbf F^X$-martingale in $\mathbb R^d$ by \cref{prop:vt-mar}, we obtain
\[
\frac{1}{2}\int_0^{t_0}\E[|v_t|^2]dt\leq\frac{t_0}{2(1-t_0)}\int_{t_0}^1\E[|v_t|^2]dt
\leq t_0H(\mu\mid\gamma_d),
\]
where the last inequality follows by \cref{ent-formula} and $t_0\leq1/2$. 
Meanwhile, we have $\sup_{t\in[0,1)}\E[|v_t|^2]\leq I(\mu\mid\gamma_d)<\infty$ by \cref{fi-formula} and \cref{lem:fi}. Hence $v_1=\lim_{t\uparrow1}v_t$ exists a.s.~by Doob's convergence theorem. 
Moreover, $\E[|v_t-v_1|^2]\to0$ as $t\uparrow1$ by \cite[Theorem 12.1]{Wi91}. In particular, $(v_t)_{t\in[0,1]}$ is an $\mathbf F^X$-martingale in $\mathbb R^d$. Consequently,
\ba{
\sum_{i=0}^{N-1}\int_{t_{i}}^{t_{i+1}}\bra{\E[|v_t|^2]-\E[|v_{t_i}|^2]}dt
&\leq\sum_{i=0}^{N-1}h_i\bra{\E[|v_{t_{i+1}}|^2]-\E[|v_{t_i}|^2]}dt\\
&\leq\bar h\bra{\E[|v_{1}|^2]-\E[|v_{0}|^2]}
\leq\bar h\lim_{t\uparrow1}\E[|v_{t}|^2]\leq \bar hI(\mu\mid\gamma_d),
}
where the last inequality follows by \cref{lem:fi}. 
All together, we complete the proof. 
\end{proof}

%

\subsection{Proof of Theorem \ref{thm:ada}}\label{proof:ada}

It suffices to consider the case $\E[|X_1|^2]<\infty$. 
Recall that $\ada_{t_i}$ are obtained via the solution to the SDE \eqref{eq:sde-ada} with $m(t,x)$ replaced by $x/t+(1-t)\score(x,t)$. 
Therefore, by an analogous argument to the proof of \eqref{eq:kl-basic-1} but changing the definition of $b$ to
\[
b(t,x)=\begin{cases}
-x(t)/(1-t) & \text{if }0\leq t<t_0,\\
\frac{x(t_{i})/t_{i}+(1-t_i)\score(x(t_{i}),t_{i})-x(t)}{1-t} & \text{if }t_{i}\leq t<t_{i+1}\text{ for some }i=0,1,\dots,N-1,\\
0 & \text{if }t_N\leq t<1,
\end{cases}
\]
we obtain
\ban{
H(P^{X_{t_N}}\mid P^{\ada_{t_N}})
&\leq\frac{1}{2}\int_0^{t_0}\E\sbra{\abs{v_t+\frac{X_t}{1-t}}^2}dt
+\sum_{i=0}^{N-1}\int_{t_{i}}^{t_{i+1}}\E\sbra{\abs{v_t-\frac{m(t_i,X_{t_i})-X_t}{1-t}}^2}dt
+\eps_\text{score}^2\notag\\
&=\frac{1}{2}\int_0^{t_0}\frac{\E[|m(t,X_t)|^2]}{(1-t)^2}dt
+\sum_{i=0}^{N-1}\int_{t_{i}}^{t_{i+1}}\frac{\E[|m(t,X_t)-m(t_i,X_{t_i})|^2]}{(1-t)^2}dt
+\eps_\text{score}^2\notag\\
&=:I+II+\eps_\text{score}^2,\label{eq:ada-basic}
}
where the second line follows by \eqref{vt-formula}. 
By \cref{lem:sl}, Jensen's inequality and $t_0\leq1/2$,
\ben{\label{eq:ada-I}
I=\frac{1}{2}\int_0^{t_0}\frac{\E[|\E[X_1\mid X_t]|^2]}{(1-t)^2}dt
\leq\frac{t_0\E[|X_1|^2]}{2(1-t_0)}\leq t_0\E[|X_1|^2].
}
Also, by \cref{lem:sl} and the Markov property of $X$,
\ban{
II&=\sum_{i=0}^{N-1}\int_{t_{i}}^{t_{i+1}}\frac{\E[|\E[X_1\mid\mcl F^X_t]-\E[X_1\mid\mcl F^X_{t_i}]|^2]}{(1-t)^2}dt
\notag\\
&=\sum_{i=0}^{N-1}\int_{t_{i}}^{t_{i+1}}\frac{\E[|\E[X_1\mid\mcl F^X_t]|^2]-\E[|\E[X_1\mid\mcl F^X_{t_i}]|^2]}{(1-t)^2}dt
\notag\\
&\leq\sum_{i=0}^{N-1}\bra{\E[|\E[X_1\mid\mcl F^X_{t_{i+1}}]|^2]-\E[|\E[X_1\mid\mcl F^X_{t_i}]|^2]}\frac{h_i}{(1-t_i)(1-t_{i+1})}
\notag\\
&=\sum_{i=0}^{N-1}\bra{A_{t_i}-A_{t_{i+1}}}\frac{h_i}{(1-t_i)(1-t_{i+1})},\label{eq:ada-II}
}
where $A_t=\E[\trace(\covariance[X_1\mid X_t])]$. 
Hence, by \eqref{eq:sampling} and \eqref{def:gamma},
\ba{
II&\leq\kappa\sum_{i=0}^{N-1}\bra{A_{t_i}-A_{t_{i+1}}}\frac{1}{1-t_i}\\
&=\kappa\sum_{i=0}^{N-1}\bra{\E[\trace(\Gamma_{t_i})]-\frac{1-t_{i+1}}{1-t_i}\E[\trace(\Gamma_{t_{i+1}})]}\\
&=\kappa\sum_{i=0}^{N-1}\frac{h_i}{1-t_i}\E[\trace(\Gamma_{t_{i+1}})]
+\kappa\E[\trace(\Gamma_{t_0})]-\kappa\E[\trace(\Gamma_{t_{N}})]\\
&\leq\kappa\sum_{i=0}^{N-1}\frac{h_i}{1-t_i}\E[\trace(\Gamma_{t_{i+1}})]
+\kappa\frac{\E[|X_1|^2]}{1-t_0}.
}
Since $t_0\leq1/2$, we conclude
\be{
H(P^{X_{t_N}}\mid P^{\ada_{t_N}})
\leq t_0\E[|X_1|^2]
+\kappa\sum_{i=0}^{N-1}\frac{h_i}{1-t_i}\E[\trace(\Gamma_{t_{i+1}})]
+2\kappa\E[|X_1|^2]
+\eps_\text{score}^2.
}
By \eqref{eq:v2}, we have for any $t\in(0,1)$
\ba{
\E[\trace(\Gamma_t)]\leq d+(1-t)\bra{\E[|X_1|^2]-d}
=td+(1-t)\E[|X_1|^2].
}
Therefore,
\ba{
&H(P^{X_{t_N}}\mid P^{\ada_{t_N}})\\
&\leq t_0\E[|X_1|^2]
+\kappa d\sum_{i=0}^{N-1}\frac{h_it_i}{1-t_i}
+\kappa(t_n-t_0)\E[|X_1|^2]
+2\kappa\E[|X_1|^2]
+\eps_\text{score}^2\\
&\leq t_0\E[|X_1|^2]
+\kappa d\int_{t_0}^{t_N}\frac{1}{1-t}dt
+3\kappa\E[|X_1|^2]
+\eps_\text{score}^2\\
&=t_0\E[|X_1|^2]
+\kappa d\log\bra{\frac{1-t_0}{\delta}}
+3\kappa\E[|X_1|^2]
+\eps_\text{score}^2.
}
Hence we complete the proof.\qed

\subsection{Proof of Theorem \ref{thm:ada-m}}\label{proof:ada-m}

The key to the proof is the following fact due to \citet{HWC26}:
\begin{lemma}[cf.~\cite{HWC26}, Lemma 10]\label{lem:hwc}
Under the assumptions of \cref{thm:ada-m}, there exists a universal constant $C\geq1$ such that
\[
\E[\trace(\covariance[X_1\mid X_t])]\leq C\frac{1-t}{t}L
\]
for all $t\in(0,1-\delta]$. 
\end{lemma}
Since this lemma is stated in a slightly more general form than the original one, we give its proof after showing that \cref{thm:ada-m} follows from this lemma. 
Here and below, for two real numbers $x,y$ we write $x\lesssim y$ if there exists a universal constant $C>0$ such that $x\leq Cy$. 

\begin{proof}[\upshape\bf Proof of \cref{thm:ada-m}]
By \eqref{eq:ada-basic} and \eqref{eq:ada-I}, we have
\ben{\label{eq:ada-m-basic}
H(P^{X_{t_N}}\mid P^{\ada_{t_N}})
\leq t_0\E[|X_1|^2]+II+\eps_\text{score}^2.
}
To bound $II$, we start from the bound \eqref{eq:ada-II}. 
Let $i_0\in\{0,1,\dots,N-1\}$ be the largest index such that $t_{i_0}\leq1/2$. 
By \eqref{eq:ada-II}, we can bound $II$ as
\ban{
II&\leq\sum_{i:0\leq i\leq i_0}\bra{A_{t_i}-A_{t_{i+1}}}\frac{h_i}{(1-t_i)(1-t_{i+1})}
+\sum_{i:i_0< i<N}\bra{A_{t_i}-A_{t_{i+1}}}\frac{h_i}{(1-t_i)(1-t_{i+1})}\notag\\
&=:II_1+II_2.\label{eq:ada-m-II}
}
First we bound $II_2$. 
If $i_0=N-1$, then $II_2=0$. Otherwise, we have $t_{i_0+1}>1/2$. Hence
\ban{
II_2&\leq\kappa\sum_{i:i_0< i<N}\bra{A_{t_i}-A_{t_{i+1}}}\frac{1}{1-t_i}
\leq\kappa\bra{\frac{A_{t_{i_0+1}}}{1-t_{i_0+1}}+\sum_{i:i_0< i<N-1}A_{t_{i+1}}\frac{h_i}{(1-t_{i+1})(1-t_i)}}
\notag\\
&\lesssim \kappa L\bra{1+\sum_{i:i_0< i<N}\frac{h_i}{1-t_i}}
\leq \kappa L\bra{1+\log\bra{\frac{1-t_0}{\delta}}},\label{eq:ada-m-II2}
}
where the first inequality follows by \ref{ass:sampling} and the third by \cref{lem:hwc}. 
Next we bound $II_1$. 
Since $\max_ih_i\leq\kappa$ by \ref{ass:sampling} and $A_{t_0}\leq\E[|X_1|^2]$, 
\ba{
II_1&\leq4\sum_{i:0\leq i\leq i_0}(A_{t_i}-A_{t_{i+1}})h_i
\leq4\kappa A_{t_0}\leq4\kappa\E[|X_1|^2].
}
Combining this with \eqref{eq:ada-m-basic}--\eqref{eq:ada-m-II2} gives \eqref{eq:ada-m}. 
If we further assume \eqref{eq:sampling-r},
\ba{
II_1&\leq24\kappa\sum_{i:0\leq i\leq i_0}\bra{A_{t_i}-A_{t_{i+1}}}t_i
\leq24\kappa \bra{A_{t_0}t_0+\sum_{i:0\leq i<i_0}A_{t_{i+1}}h_i}\\
&\lesssim \kappa L\bra{1+\sum_{i:0\leq i<i_0}\frac{h_i}{t_i}}
\leq \kappa L\bra{1+\log(1/t_0)}.
}
Hence we can replace $\kappa\E[|X_1|^2]$ in \eqref{eq:ada-m} by $C\kappa L\log(1/t_0)$ for some universal constant $C$.
\end{proof}


Now we turn to the proof of \cref{lem:hwc}. 
The proof strategy is the same as \cite[Lemma 10]{HWC26}. 
We first introduce some notation. 
We set $m:=\mcl N(\mcl X,\eps_0)$. 
Note that $\log m\leq k\log(1/\eps_0)$ by assumption. 
Also, by definition we can take an $\eps_0$-net $\{x^*_1,\dots,x^*_m\}$ for $\mcl X$. 
Note that $x^*_1,\dots,x^*_m$ are distinct due to the maximality of $m$. 
We fix a constant $\rho\in(0,1/3]$ whose value is specified in the main body of the proof of \cref{lem:hwc}. 
We write $\mcl B_i=\mcl B(x_i^*;\eps_0)$ for $i=1,\dots,m$ and set 
\[
\mcl I:=\{i\in\{1,\dots,m\}:\mu(\mcl B_i)\geq \rho/m\}.
\]
The following three lemmas are counterparts of Lemmas B.3 and B.4 in \cite{HWC26}. 
\begin{lemma}\label{lem:prob-lb}
$P(X_1\notin\bigcup_{i\in\mcl I}\mcl B_i)\leq\rho$. 
\end{lemma}

\begin{proof}
Since $\{\mcl B_1,\dots,\mcl B_m\}$ is a covering of $\mcl X$, $x\in\mcl X\setminus\bigcup_{i\in\mcl I}\mcl B_i$ implies $x\in\bigcup_{i\in\mcl I^c}\mcl B_i$. 
Hence
\ba{
P\bra{X_1\notin\bigcup_{i\in\mcl I}\mcl B_i}
\leq P\bra{X_1\in\bigcup_{i\in\mcl I^c}\mcl B_i}
\leq\sum_{i\in\mcl I^c}P(X_1\in\mcl B_i)\leq \rho.
}
\end{proof}

\begin{lemma}\label{lem:z-modulus}
Set $Z_t=\beta_t/\sqrt{t(1-t)}$. 
There exists a universal constant $C_0\geq1$ such that
\ben{\label{eq:z-modulus}
|Z_t\cdot(x-y)|\leq C_0\bra{L_\rho\sqrt\delta+\sqrt{L_\rho}|x-y|}\quad\text{for any }x,y\in\mcl X
}
with probability at least $1-\rho$, where $L_\rho:=k\log(1/\eps_0)+\log(1/\rho)$.   
\end{lemma}

\begin{proof}
Observe that $Z_t\sim N(0,1)$. 
Hence, by the Borell--TIS inequality (see e.g.~Proposition A.2.1 in \cite{VaWe23}) and \eqref{eq:gw}, we have for any $u\geq0$
\ba{
\sup_{x,y\in\mcl X:|x-y|\leq\eps_0}|(x-y)\cdot Z_t|
&\leq\E\sbra{\sup_{x,y\in\mcl X:|x-y|\leq\eps_0}|(x-y)\cdot Z_t|}+u\sup_{x,y\in\mcl X:|x-y|\leq\eps_0}\sqrt{\E\sbra{|(x-y)\cdot Z_t|^2}}\\
&\leq \sqrt\delta k\log(1/\eps_0) +u\eps_0
}
with probability at least $1-2e^{-u^2/2}$. 
Taking $u=\sqrt{4\log(1/\rho)}$ gives
\ben{\label{eq:sup-g}
\sup_{x,y\in\mcl X:|x-y|\leq\eps_0}|(x-y)\cdot Z_t|\leq \sqrt\delta k\log(1/\eps_0) +\eps_0\sqrt{4\log(1/\rho)}
}
with probability at least $1-\rho/2$. 

Next, for any $i<j$, we have $\{Z_t\cdot(x^*_i-x^*_j)\}/|x^*_i-x^*_j|\sim N(0,1)$, so 
\[
P(|Z_t\cdot(x^*_i-x^*_j)|\geq u|x^*_i-x^*_j|)\leq e^{-u^2/2}
\]
for all $u\geq0$. Thus, setting $u=\sqrt{2\log(m^2/\rho)}$, we obtain by the union bound
\ba{
P\bra{|Z_t\cdot(x^*_i-x^*_j)|\geq u|x^*_i-x^*_j|\text{ for some }1\leq i<j\leq m}
\leq\binom{m}{2}e^{-u^2/2}\leq\frac{\rho}{2}.
}
Since $u\lesssim \sqrt{L_\rho}$, we have
\ben{\label{eq:max-g}
|Z_t\cdot(x^*_i-x^*_j)|\lesssim \sqrt{L_\rho}|x^*_i-x^*_j|\quad\text{for any }1\leq i<j\leq m
}
with probability at least $1-\rho/2$. 

Finally, for any $x,y\in\mcl X$, there exist $i,j\in\{1,\dots,m\}$ such that $x\in\mcl B_i$ and $y\in\mcl B_j$. Then, on the event that \eqref{eq:sup-g} and \eqref{eq:max-g} jointly holds, 
\ba{
|Z_t\cdot(x-y)|
&\leq|Z_t\cdot(x-x_i^*)|+|Z_t\cdot(x^*_i-x^*_j)|+|Z_t\cdot(x_j^*-y)|\\
&\lesssim \sqrt\delta k\log(1/\eps_0) +\eps_0\sqrt{\log(1/\rho)}+\sqrt{L_\rho}|x^*_i-x^*_j|\\
&\leq \sqrt\delta k\log(1/\eps_0)+3\sqrt{L_\rho}\eps_0+\sqrt{L_\rho}|x-y|\\
&\leq4L_\rho\sqrt{\delta}+\sqrt{L_\rho}|x-y|.
}
This completes the proof. 
\end{proof}

\begin{lemma}\label{lem:hwc-lem11}
With $C_0$ the constant in \cref{lem:z-modulus}, define a random subset $\mcl E_t\subset\mcl X$ as
\[
\mcl E_t:=\cbra{y\in\mcl X:|y-X_1|\geq5C_0\sqrt{L_\rho\frac{1-t}{t}}}.
\]
Then, with probability at least $1-2\rho$,
\ben{\label{eq:hwc-lem11}
\mu_{t,X_t}(\mcl E_t)\leq\exp\bra{-3C_0L_\rho}.
}
\end{lemma}

\begin{proof}
Let $\Omega_0$ be the event on which $X_1\in\bigcup_{i\in\mcl I}\mcl B_i$ and \eqref{eq:z-modulus} jointly hold. By Lemmas \ref{lem:prob-lb} and \ref{lem:z-modulus}, $P(\Omega_0)\geq1-2\rho$. 
Hence we complete the proof once we show that \eqref{eq:hwc-lem11} holds on $\Omega_0$. 

Observe that for every $y\in\mathbb R^d$
\ba{
\frac{|y|^2}{2}-\frac{|y-X_t|^2}{2(1-t)}
&=-\frac{t|y|^2-2X_t\cdot y+|X_t|^2}{2(1-t)}
=-\frac{t|y-X_1|^2-2\beta_t\cdot y-t|X_1|^2+|X_t|^2}{2(1-t)}\\
&=-\frac{U_t(y)-2\beta_t\cdot X_1-t|X_1|^2+|X_t|^2}{2(1-t)},
}
where $U_t(y):=t|y-X_1|^2-2\beta_t\cdot (y-X_1)$. 
Hence we can rewrite $\mu_{t,X_t}(\mcl E_t)$ as
\ben{\label{hwc-lem11-aim1}
\mu_{t,X_t}(\mcl E_t)=\frac{\int_{\mcl E_t}\exp\bra{-\frac{U_t(y)}{2(1-t)}}\mu(dy)}{\int_{\mathbb R^d}\exp\bra{-\frac{U_t(y)}{2(1-t)}}\mu(dy)}.
}
Now, if $y\in\mcl E_t$, we have on $\Omega_0$
\ba{
U_t(y)&\geq t|y-X_1|^2-2C_0\sqrt{t(1-t)}\bra{L_\rho\sqrt\delta+\sqrt{L_\rho}|y-X_1|}\quad(\because\eqref{eq:z-modulus})\\
&=t|y-X_1|\bra{|y-X_1|-2C_0\sqrt{L_\rho\frac{1-t}{t}}}-2C_0\sqrt{t(1-t)}L_\rho\sqrt\delta\\
&\geq 3C_0\sqrt{L_\rho t(1-t)}|y-X_1|-2C_0\sqrt{t(1-t)}L_\rho\sqrt\delta
\quad(\because y\in\mcl E_t)\\
&\geq C_0\sqrt{L_\rho t(1-t)}\bra{3|y-X_1|-2\sqrt{L_\rho(1-t)}}
\quad(\because \delta\leq1-t)\\
&\geq 13C_0L_\rho(1-t)\quad(\because y\in\mcl E_t).
}
Hence
\ben{\label{hwc-lem11-aim2}
\int_{\mcl E_t}\exp\bra{-\frac{U_t(y)}{2(1-t)}}\mu(dy)\leq\exp\bra{-\frac{13}{2}C_0L_\rho}.
}
Meanwhile, on the event $\Omega_0$ we can take an $i^*\in\mcl I$ so that $X_1\in\mcl B_{i^*}$. 
Then, if $y\in\mcl B_{i^*}$, 
\ba{
U_t(y)
&\leq t\eps_0^2+2\sqrt{t(1-t)}\cdot C_0\bra{L_\rho\sqrt\delta+\sqrt{L_\rho}\eps_0}
\leq 5C_0L_\rho(1-t),
}
where the first inequality follows by \eqref{eq:z-modulus} and the second by $\eps_0\leq \sqrt{L_\rho\delta}$ and $\delta\leq1-t$. 
Further, since $i^*\in\mcl I$,
\ba{
\mu(\mcl B_{i^*})\geq\exp\bra{-\log(m/\rho)}
\geq\exp(-L_\rho)\geq\exp(-C_0L_\rho).
}
Consequently,
\besn{\label{hwc-lem11-aim3}
\int_{\mathbb R^d}\exp\bra{-\frac{U_t(y)}{2(1-t)}}\mu(dy)
&\geq\int_{\mcl B_{i^*}}\exp\bra{-\frac{U_t(y)}{2(1-t)}}\mu(dy)\\
&\geq\mu(\mcl B_{i^*})\exp\bra{-\frac{5C_0L_\rho}{2}}
\geq\exp\bra{-3C_0L_\rho}.
}
Combining \eqref{hwc-lem11-aim1}--\eqref{hwc-lem11-aim3} gives \eqref{eq:hwc-lem11} on $\Omega_0$. 
\end{proof}

\begin{proof}[\bf\upshape Proof of \cref{lem:hwc}]
Define the random set $\mcl E_t$ in \cref{lem:hwc-lem11} with $\rho=\delta^2/R^4$. 
Note that $L_\rho\lesssim L$ in this case. Also, \cref{lem:hwc-lem11} gives
\ben{\label{eq:mut-est}
\E[\mu_{t,X_t}(\mcl E_t)]\leq\exp(-3C_0L_\rho)+2\rho\leq3\rho.
}

Now, observe that \cref{lem:sl} gives
\[
\trace(\covariance[X_1\mid X_t])=\int_{\mathbb R^d}|y-m(t,X_t)|^2\mu_{t,X_t}(dy)
\leq\int_{\mathbb R^d}|y-X_1|^2\mu_{t,X_t}(dy).
\]
Therefore, using the definition of $\mcl E_t$, we deduce
\ba{
\trace(\covariance[X_1\mid X_t])
&\leq\int_{\mcl E_t}|y-X_1|^2\mu_{t,X_t}(dy)
+\int_{\mcl E_t^c}|y-X_1|^2\mu_{t,X_t}(dy)\\
&\leq 2\int_{\mcl E_t}|y-\E[X_1]|^2\mu_{t,X_t}(dy)+2|X_1-\E[X_1]|^2\mu_{t,X_t}(\mcl E_t)+25C_0^2\frac{1-t}{t}L_\rho.
}
By the Schwarz inequality and \cref{lem:sl},
\ba{
\E\sbra{\int_{\mcl E_t}|y-\E[X_1]|^2\mu_{t,X_t}(dy)}
&\leq\sqrt{\E\sbra{\int_{\mathbb R^d}|y-\E[X_1]|^4\mu_{t,X_t}(dy)}\E[\mu_{t,X_t}(\mcl E_t)]}\\
&=\sqrt{\E[|X_1-\E[X_1]|^4]\E[\mu_{t,X_t}(\mcl E_t)]}
}
and
\[
\E[|X_1-\E[X_1]|^2\mu_{t,X_t}(\mcl E_t)]\leq\sqrt{\E[|X_1-\E[X_1]|^4]\E[\mu_{t,X_t}(\mcl E_t)]}.
\]
Combining these estimates with \eqref{eq:mut-est} gives
\ba{
\E[\trace(\covariance[X_1\mid X_t])]
&\lesssim R^2\sqrt{\rho}+\frac{1-t}{t}L
=\delta+\frac{1-t}{t}L\leq\frac{1-t}{t}L.
}
Hence we obtain the desired result.
\end{proof}

\paragraph{Acknowledgments}

This work was partly supported by JST CREST Grant Number JPMJCR2115 and JSPS KAKENHI Grant Number JP24K14848.

{\small
\renewcommand*{\baselinestretch}{1}\selectfont
\addcontentsline{toc}{section}{References}

}

\end{document}